\documentclass{svjour3}
\journalname{Machine learning}

\usepackage[hidelinks]{hyperref}
\usepackage[bitstream-charter]{mathdesign}

\usepackage{amsmath}
\usepackage{inconsolata}
\usepackage{pgfplots}
\usepackage{algorithm}
\usepackage{listings}
\usepackage[normalem]{ulem}
\usepackage[utf8]{inputenc}

\usepackage{tikz}
\usetikzlibrary{positioning}
\usetikzlibrary{calc}

\usepackage{varwidth}
\usepackage{caption}
\usepackage{booktabs}
\usepackage{subcaption}
\usepackage{mathtools}
\lstnewenvironment{myalgorithm}[1][] 
{
    \lstset{ 
        mathescape=true,
        numbers=left,
        escapeinside={*}{*},
        columns=flexible,
        numberstyle=\ttfamily,
        basicstyle=\ttfamily,
        keywordstyle=\bfseries\ttfamily,
        keywords={,and, return, not, def, in, if, else, for, foreach, while, }
        numbers=left,
        xleftmargin=.02\textwidth,
    }
}
{}

\newcommand{\cpl}[1]{
\begin{center}
\tw{#1}
\end{center}
}

\newcommand{\dilp}{$\partial$ILP}
\newcommand{\name}{Popper}
\newcommand{\alan}{Enumerate}
\newcommand{\tw}[1]{\texttt{#1}}

\newcommand{\pall}{\textbf{P\textsubscript{all}}}
\newcommand{\psome}{\textbf{P\textsubscript{some}}}
\newcommand{\pnone}{\textbf{P\textsubscript{none}}}

\newcommand{\nsome}{\textbf{N\textsubscript{some}}}
\newcommand{\nnone}{\textbf{N\textsubscript{none}}}

\usepackage{xcolor}

\title{Learning programs by learning from failures}

\author{Andrew Cropper \and Rolf Morel}

\institute{
A. Cropper\at
University of Oxford\\
\email{andrew.cropper@cs.ox.ac.uk}
\and
R. Morel\at
University of Oxford\\
\email{rolf.morel@cs.ox.ac.uk}
}

\begin{document}

\maketitle

\begin{abstract}
We describe an inductive logic programming (ILP) approach called \emph{learning from failures}.
In this approach, an ILP system (the learner) decomposes the learning problem into three separate stages: \emph{generate}, \emph{test}, and \emph{constrain}.
In the generate stage, the learner generates a hypothesis (a logic program) that satisfies a set of \emph{hypothesis constraints} (constraints on the syntactic form of hypotheses).
In the test stage, the learner tests the hypothesis against training examples.
A hypothesis \emph{fails} when it does not entail all the positive examples or entails a negative example.
If a hypothesis fails, then, in the constrain stage, the learner learns constraints from the failed hypothesis to prune the hypothesis space, i.e. to constrain subsequent hypothesis generation.
For instance, if a hypothesis is too general (entails a negative example), the constraints prune generalisations of the hypothesis.
If a hypothesis is too specific (does not entail all the positive examples), the constraints prune specialisations of the hypothesis.
This loop repeats until either (i) the learner finds a hypothesis that entails all the positive and none of the negative examples, or (ii) there are no more hypotheses to test.
We introduce \name{}, an ILP system that implements this approach by combining answer set programming and Prolog.
\name{} supports infinite problem domains, reasoning about lists and numbers, learning textually minimal programs, and learning recursive programs.
Our experimental results on three domains (toy game problems, robot strategies, and list transformations) show that (i) constraints drastically improve learning performance, and (ii) \name{} can outperform existing ILP systems, both in terms of predictive accuracies and learning times.
\end{abstract}
\section{Introduction}
\label{sec:intro}

Inductive logic programming (ILP) \cite{mugg:ilp} is a form of machine learning.
Given examples of a target predicate and background knowledge (BK), the ILP problem is to induce a hypothesis which, with the BK, correctly generalises the examples.
A key characteristic of ILP is that it represents the examples, BK, and hypotheses as logic programs (sets of logical rules).

Compared to most machine learning approaches, ILP has several advantages \cite{ilp30}.
ILP systems can generalise from small numbers of examples, often a single example \cite{mugg:metabias}.
Because hypotheses are logic programs, they can be read by humans, crucial for explainable AI and ultra-strong machine learning \cite{usml}.
Finally, because of their symbolic nature, ILP systems naturally support lifelong and transfer learning \cite{crop:forget}, which is considered essential for human-like AI \cite{lake:ai}.

The fundamental problem in ILP is to efficiently search a large hypothesis space (the set of all hypotheses).
A popular ILP approach is to use a set covering algorithm to learn hypotheses one clause at-a-time \cite{foil,progol,tilde,aleph,atom}.
Systems that implement this approach are often efficient because they are example-driven.
However, these systems tend to learn overly specific solutions and struggle to learn recursive programs \cite{hyper,ilp30}.
An alternative, but increasingly popular, approach is to encode the ILP problem as an answer set programming (ASP) problem \cite{aspal,ilasp,inspire,hexmil,apperception}.
Systems that implement this approach can often learn optimal and recursive programs and can harness state-of-the-art ASP solvers, but often struggle with scalability, especially in terms of the problem domain size.

In this paper, we describe an ILP approach called \emph{learning from failures} (LFF).
In this approach, the learner (an ILP system) decomposes the ILP problem into three separate stages: \emph{generate}, \emph{test}, and \emph{constrain}.
In the generate stage, the learner generates a hypothesis (a logic program) that satisfies a set of \emph{hypothesis constraints} (constraints on the syntactic form of hypotheses).
In the test stage, the learner tests a hypothesis against training examples.
A hypothesis \emph{fails} when it does not entail all the positive examples or entails a negative example.
If a hypothesis fails, then, in the constrain stage, the learner learns hypothesis constraints from the failed hypothesis to prune the hypothesis space, i.e.~to constrain subsequent hypothesis generation.

Compared to other approaches that employ a generate/test/constrain loop \cite{law:thesis}, a key idea in this paper is to use theta-subsumption \cite{plotkin:thesis} to translate a failed hypothesis into a set of constraints.
For instance, if a hypothesis is too general (entails a negative example), the constraints prune generalisations of the hypothesis.
If a hypothesis is too specific (does not entail all the positive examples), the constraints prune specialisations of the hypothesis.
This loop repeats until either (i) the learner finds a \emph{solution} (a hypothesis that entails all the positive examples and none of the negative examples), or (ii) there are no more hypotheses to test.
Figure \ref{fig:loop} illustrates this loop.
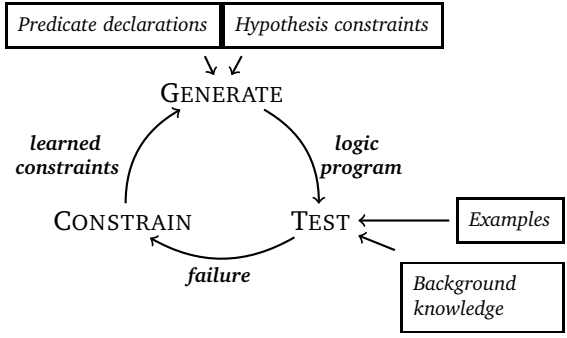
\begin{figure}[ht]
\centering
\small
\begin{tikzpicture}
\node (Generate) at (1.3cm, 1.7cm) {\large \scshape Generate};
\node (Test) at (2.6cm, 0cm) {\large \scshape Test};
\node (Constrain) at (0cm, 0cm) {\large \scshape Constrain};

\node[above left=0.25cm and -1cm of Generate] (Preds) [align=right]
{\framebox{\em Predicate declarations\vphantom{Hp}}};

\node[above right=0.25cm and -1cm of Generate] (HypothesisConstraints) [align=left]
{\framebox{\em Hypothesis constraints}};

\node[right = 2.8cm of Test,anchor=east] (Examples)
{\framebox{\em Examples}};
\node[below right = 0.8cm and 2.8cm of Test,anchor=east] (BK)
{\framebox{\begin{minipage}{1.8cm}
\em
Background
knowledge
\end{minipage}}};

\draw [->,thick] (Preds)++(1.2cm,-0.4cm) to (Generate);
\draw [->,thick] (HypothesisConstraints)++(-1.2cm,-0.4cm) to (Generate);
\draw [->,thick] (Examples) to (Test);
\draw [->,thick] (BK)++(-1.15cm,0.63cm) to (Test);

\draw [->,thick]
(Generate)
to[bend left]
node[above right = -0.50cm and -0.4cm of Test,text width=3cm,align=left] {
\textbf{
{\em
\hspace{0.6cm}logic\\
\hspace{0.5cm}program
}
}}
(Test);

\draw [->,thick]
(Test)
($ (Test.south) - (0.35,0) $)
to[bend left]
node[above=1pt,align=center] {
}
node[below,align=center] {
\textbf{\emph{failure}}
}
($ (Constrain.south) + (0.35,0) $);

\draw [->, thick] (Constrain) edge[bend left]
node[above left = -0.45cm and -0.25cm of Constrain,text width=2.4cm,align=right] {
\textbf{
{\em
learned\hspace{0.65cm}~\\
constraints\hspace{0.4cm}~}\\
}
} (Generate);
\end{tikzpicture}
\caption{
    The generate, test, and constrain loop.
}
\label{fig:loop}
\end{figure}

\begin{example}[Learning from failures]
To illustrate our approach, consider learning a \emph{last/2} hypothesis to find the last element of a list.
For simplicity, assume an initial hypothesis space $\mathcal{H}_1$:
\[
\mathcal{H}_1 = \left\{
\begin{array}{l}
\tw{h$_1$} = \left\{
\begin{array}{l}
    \tw{last(A,B):- head(A,B).}\\
\end{array}
\right\}\\
\tw{h$_2$} = \left\{
\begin{array}{l}
    \tw{last(A,B):- head(A,B),empty(A).}\\
\end{array}
\right\}\\
\tw{h$_3$} = \left\{
\begin{array}{l}
    \tw{last(A,B):- tail(A,C),head(C,B).}\\
\end{array}
\right\}\\
\tw{h$_4$} = \left\{
\begin{array}{l}
    \tw{last(A,B):- reverse(A,C),head(C,B).}\\
\end{array}
\right\}\\
\tw{h$_5$} = \left\{
\begin{array}{l}
    \tw{last(A,B):- head(A,B),reverse(A,C),head(C,B).}\\
\end{array}
\right\}\\
\tw{h$_6$} = \left\{
\begin{array}{l}
    \tw{last(A,B):- tail(A,C),head(C,B).}\\
    \tw{last(A,B):- reverse(A,C),head(C,B).}\\
\end{array}
\right\}\\
\tw{h$_7$} = \left\{
\begin{array}{l}
    \tw{last(A,B):- tail(A,C),head(C,B).}\\
    \tw{last(A,B):- tail(A,C),tail(C,D),head(D,B).}\\
\end{array}
\right\}\\
\tw{h$_8$} = \left\{
\begin{array}{l}
    \tw{last(A,B):- reverse(A,C),tail(C,D),head(D,B).}\\
    \tw{last(A,B):- tail(A,C),reverse(C,D),head(D,B).}\\
\end{array}
\right\}
\end{array}
\right\}
\]

\noindent
Also assume we have the positive ($E^+$) and negative ($E^-$) examples:

\[
E^+ = \left\{
\begin{array}{l}
\tw{last([l,a,u,r,a],a).}\\
\tw{last([p,e,n,e,l,o,p,e],e).}
\end{array}
\right\}
\hspace{3ex}
E^- = \left\{
\begin{array}{l}
\tw{last([e,m,m,a],m).}\\
\tw{last([j,a,m,e,s],e).}
\end{array}
\right\}
\]

\noindent
In the generate stage, the learner generates a hypothesis:
\[
\tw{h$_1$} = \left\{
\begin{array}{l}
    \tw{last(A,B):- head(A,B).}\\
\end{array}
\right\}
\]
\noindent
In the test stage, the learner tests \tw{h$_1$} against the examples and finds that it \emph{fails} because it does not entail any positive example and is therefore too \emph{specific}.
In the constrain stage, the learner learns hypothesis constraints to prune specialisations of \tw{h$_1$} (\tw{h$_2$} and \tw{h$_5$}) from the hypothesis space.
The hypothesis space is now:
\[
\mathcal{H}_2 = \left\{
\begin{array}{l}
\tw{h$_3$} = \left\{
\begin{array}{l}
    \tw{last(A,B):- tail(A,C),head(C,B).}\\
\end{array}
\right\}\\
\tw{h$_4$} = \left\{
\begin{array}{l}
    \tw{last(A,B):- reverse(A,C),head(C,B).}\\
\end{array}
\right\}\\
\tw{h$_6$} = \left\{
\begin{array}{l}
    \tw{last(A,B):- tail(A,C),head(C,B).}\\
    \tw{last(A,B):- reverse(A,C),head(C,B).}\\
\end{array}
\right\}\\
\tw{h$_7$} = \left\{
\begin{array}{l}
    \tw{last(A,B):- tail(A,C),head(C,B).}\\
    \tw{last(A,B):- tail(A,C),tail(C,D),head(D,B).}\\
\end{array}
\right\}\\
\tw{h$_8$} = \left\{
\begin{array}{l}
    \tw{last(A,B):- reverse(A,C),tail(C,D),head(D,B).}\\
    \tw{last(A,B):- tail(A,C),reverse(C,D),head(D,B).}\\
\end{array}
\right\}
\end{array}
\right\}
\]
\noindent
In the next generate stage, the learner generates another hypothesis:
\[
\tw{h$_3$} = \left\{
\begin{array}{l}
    \tw{last(A,B):- tail(A,C),head(C,B).}\\
\end{array}
\right\}
\]
\noindent
The learner tests \tw{h$_3$} against the examples and finds that it fails because it entails the negative example \tw{last([e,m,m,a],m)} and is therefore too \emph{general}.
The learner learns constraints to prune generalisations of \tw{h$_3$} (\tw{h$_6$} and \tw{h$_7$}) from the hypothesis space.
The hypothesis space is now:
\[
\mathcal{H}_3 = \left\{
\begin{array}{l}
\tw{h$_4$} = \left\{
\begin{array}{l}
    \tw{last(A,B):- reverse(A,C),head(C,B).}\\
\end{array}
\right\}\\
\tw{h$_8$} = \left\{
\begin{array}{l}
    \tw{last(A,B):- reverse(A,C),tail(C,D),head(D,B).}\\
    \tw{last(A,B):- tail(A,C),reverse(C,D),head(D,B).}\\
\end{array}
\right\}
\end{array}
\right\}
\]
\noindent
The learner generates another hypothesis (\tw{h$_4$}), tests it against the examples, finds that it does not fail, and returns it.
\end{example}

Whereas many ILP approaches iteratively refine a clause \cite{foil,progol,claudien,tilde,aleph,atom} or refine a hypothesis \cite{mis,hyper,raspal,metagol}, our approach refines the \emph{hypothesis space} through learned \emph{hypothesis constraints}.
In other words, LFF continually builds a set of constraints.
The more constraints we learn, the more we reduce the hypothesis space.
By reasoning about the hypothesis space, our approach can drastically prune large parts of the hypothesis space by testing a single hypothesis.

We implement our approach in \name{}\footnote{
    \name{} is named after \emph{Karl Poppper}, whose idea of \emph{falsification} \cite{popper2005logic} inspired our approach, as it did Shapiro's MIS approach \cite{mis}.
    In fact, one can view our approach as Popper's idea of falsification, where a \emph{failure} is a refutation/falsification.
    In other words, in our approach, a learner \emph{deduces} what hypotheses \emph{cannot} be true and prunes them from the hypothesis space, leaving only hypotheses not yet refuted.
}, a new ILP system which combines ASP and Prolog.
In the generate stage, \name{} uses ASP to declaratively define, constrain, and search the hypothesis space.
The idea is to frame the problem as an ASP problem where an answer set (a model) corresponds to a program, an approach also employed by other recent ILP approaches  \cite{aspal,ilasp,hexmil,inspire}.
By later learning hypothesis constraints, we eliminate answer sets and thus prune the hypothesis space.
Our first motivation for using ASP is its declarative nature, which allows us to, for instance, define constraints to enforce Datalog and type restrictions, constraints to prune recursive hypotheses that do not contain base cases, and constraints to prune generalisations and specialisations of a failed hypothesis.
Our second motivation is to use state-of-the-art ASP systems \cite{clingo} to efficiently solve our complex constraint problem.
In the test stage, \name{} uses Prolog to test hypotheses against the examples and BK.
Our main motivation for using Prolog in this stage is to learn programs that use lists, numbers, and large domains.
In the constrain stage, \name{} learns hypothesis constraints (in the form of ASP constraints) from failed hypotheses to prune the hypothesis space, i.e.~to constrain subsequent hypothesis generation.
To efficiently combine the three stages, \name{} uses ASP's multi-shot solving \cite{multishot} to maintain state between the three stages, e.g.~to remember learned conflicts on the hypothesis space.

To give a clear overview of \name{}, Table \ref{tab:diffs} compares \name{} to Aleph \cite{aleph}, a classical ILP system, and Metagol \cite{metagol}, ILASP3 \cite{law:thesis}, and \dilp{} \cite{dilp}, three state-of-the-art ILP systems based on Prolog, ASP, and neural networks respectively.
Compared to Aleph, \name{} can learn optimal and recursive programs\footnote{
    Aleph can learn recursive programs but struggles because it requires examples of both the base and inductive cases.
}.
Compared to Metagol, \name{} does not need metarules \cite{crop:reduce}, so can learn programs with any arity predicates.
Compared to \dilp{}, \name{} supports non-ground clauses as BK, so supports large and infinite domains.
Compared to ILASP3, \name{} does not need to ground a program, so scales better as the domain size grows (Section \ref{sec:robots}).
Compared to all the systems, \name{} supports hypothesis constraints, such as disallowing the co-occurrence of predicate symbols in a program, disallowing recursive hypotheses that do not contain base cases, or preventing subsumption redundant hypotheses.

\begin{table}[ht]
\centering
\small
\begin{tabular}{@{}llllll@{}}
\toprule
& \textbf{Aleph} & \textbf{Metagol} & \textbf{ILASP3} & \textbf{\dilp{}} & \textbf{\name{}}\\
\midrule
\textbf{Hypotheses} & Normal & Definite & ASP & Datalog & Definite\\
\textbf{Language bias} & Modes & Metarules & Modes & Templates  & Declarations\\
\textbf{Predicate invention} & No & \textbf{Yes} & Partly & Partly & No\\
\textbf{Noise handling} & \textbf{Yes} & No & \textbf{Yes} & \textbf{Yes} & No\\
\textbf{Recursion} & Partly & \textbf{Yes} & \textbf{Yes} & \textbf{Yes} & \textbf{Yes}\\
\textbf{Optimality} & No & \textbf{Yes} & \textbf{Yes} & \textbf{Yes} & \textbf{Yes}\\
\textbf{Hypothesis constraints} & No & No & No & No & \textbf{Yes}\\
\bottomrule
\end{tabular}
\caption{
    A simplified comparison of ILP systems.
    Aleph can learn recursive programs but struggles because it requires examples of both the base and inductive cases.
    Metagol supports \emph{automatic} predicate invention, whereas ILASP3 and \dilp{} support \emph{prescriptive} predicate invention \cite{law:thesis}, where the arity and argument types of an invented predicate must be specified by the given language bias.
}
\label{tab:diffs}
\end{table}

ILASP3 \cite{law:thesis} is the most similar ILP approach and also employs a generate/test/constrain loop.
We discuss in detail the differences between ILASP3 and \name{} in Section \ref{sec:asp} but briefly summarise them now.
ILASP3 learns ASP programs and can handle noise, whereas \name{} learns Prolog programs and cannot currently handle noise.
ILASP3 pre-computes every rule in the hypothesis space and therefore struggles to learn rules with many body literals (Section \ref{sec:buttons}).
By contrast, \name{} does not pre-compute every rule, which allows it to learn rules with many body literals.
With each iteration, ILASP3 finds the best hypothesis it can.
If the hypothesis does not cover one of the examples, ILASP3 finds a reason why and then generates constraints to guide subsequent search\footnote{
This statement covers the noiseless ILASP3 setting. Things are slightly more complicated in noisy tasks where examples are given penalties and ILASP3 may return a hypothesis that does not cover all examples, but is optimal with respect to the penalties.
}.
The constraints are boolean formulas over the rules in the hypothesis space, an approach that requires a set of pre-computed rules and the computation of which can be very expensive.
Another way of viewing ILASP3 is that it uses a counter-example guided \cite{cegis} approach and translates an uncovered example $e$ into a constraint that is satisfied if and only if $e$ is covered.
By contrast, the key idea of \name{} is that when a hypothesis fails, \name{} uses theta-subsumption \cite{plotkin:thesis} to translate the \emph{hypothesis} itself into a set of \emph{hypothesis constraints} to rule out generalisations and specialisations of it, which does not need a set of pre-computed rules and which is substantially quicker to compute.


Overall our specific contributions in this paper are:

\begin{itemize}
\setlength\itemsep{0pt}
\setlength\parskip{0pt}
\item We define the LFF problem, determine the size of the LFF hypothesis space, define hypothesis \emph{generalisations} and \emph{specialisations} based on theta-subsumption and show that they are sound with respect to optimal solutions (Section \ref{sec:setting}).
\item We introduce \name{}, an ILP system that learns definite programs (Section \ref{sec:impl}).
\name{} support types, learning optimal (textually minimal) solutions, learning recursive programs, reasoning about lists and infinite domains, and hypothesis constraints.
\item We experimentally show (Section \ref{sec:experiments}) on three domains (toy game problems, robot strategies, and list transformations) that (i) constraints drastically reduce the hypothesis space, (ii) \name{} scales well with respect to the optimal solution size, the number of background relations, the domain size, the number of training examples, and the size of the training examples, and (iii) \name{} can substantially outperform existing ILP systems both in terms of predictive accuracies and learning times.
\end{itemize}
\section{Related work}
\label{sec:related}

\subsection{Inductive program synthesis}
The goal of inductive program synthesis is to induce a program from a partial specification, typically input/output examples \cite{mis}.
This topic interests researchers from many areas of computer science, notably machine learning (ML) and programming languages (PL).
The major\footnote{
    Minor differences include the form of specification and noise handling.
} difference between ML and PL approaches is the generality of solutions (synthesised programs).
PL approaches often aim to find \emph{any} program that fits the specification, regardless of whether it generalises.
Indeed, PL approaches rarely evaluate the ability of their systems to synthesise solutions that generalise, i.e.~they do not measure predictive accuracy \cite{lambdasquared,DBLP:conf/pldi/PolikarpovaKS16,awscp,DBLP:conf/pldi/FengMBD18,prosynth}.
By contrast, the major challenge in ML is learning hypotheses that \emph{generalise} to unseen examples.
Indeed, it is often trivial for an ML system to learn an overly specific solution for a given problem.
For instance, an ILP system can trivially construct the bottom clause \cite{progol} for each example.
Because of this major difference, in the rest of this section, we focus on ML approaches to inductive program synthesis.
We first, however, briefly cover two PL approaches, which share similarities to our learning from failures idea.

Neo \cite{DBLP:conf/pldi/FengMBD18} synthesises non-recursive programs using SMT encoded properties and a three staged loop.
Neo inherently requires SMT encoded properties for domain specific functions (i.e.~its background knowledge).
For instance, their property for \emph{head}, taking an \tw{$input$} list and returning an \tw{$output$} list, is the formula \tw{$input.size \ge 1 \land output.size = 1 \land output.max \leq input.max$}.
Neo's first stage builds up partially constructed programs.
Its second stage uses SMT-based deduction on the properties of a partial program to detect inconsistency.
The third stage determines related partial programs who must be inconsistent and can therefore be pruned.
As it typically uses over-approximate properties, Neo can fail to detect inconsistency with the examples, in which case no programs get pruned.
In contrast, our approach does not need any properties of background predicates.
We only check whether a hypothesis entails the examples, always pruning specialisations and/or generalisations when the hypothesis fails.
Neo cannot synthesise recursive programs, nor is it guaranteed to synthesise optimal (textually minimal) programs.
By contrast, \name{} can learn optimal and recursive logic programs.

ProSynth \cite{prosynth} takes as input a set of candidate Datalog rules and returns a subset of them.
ProSynth learns constraints that disallow certain clause combinations, e.g.~to prevent clauses that entail a negative example from occurring together.
\name{} differs from ProSynth in several ways.
ProSynth takes as input the full hypothesis space (the set of candidate rules).
By contrast, \name{} does not fully construct the hypothesis space.
This difference is important because it is often infeasible to pre-compute the full hypothesis space.
For instance, the largest number of candidate rules considered in the ProSynth experiments is 1000.
By contrast, in our first two experiments (Section \ref{sec:buttons}), the hypothesis spaces contain approximately $10^6$ and $10^{16}$ rules.
ProSynth provides no guarantees about solution size.
By contrast, \name{} is guaranteed to learn an optimal (smallest) solution (Theorem \ref{thm:optimal}).
Moreover, whereas ProSynth synthesises Datalog programs, \name{} additionally learns definite programs, and thus supports learning programs with infinite domains.

\subsection{Inductive logic programming}

There are various ML approaches to inductive program synthesis, including neural approaches \cite{deepcoder,ellis:scc,ellis:repl}.
We focus on inductive logic programming (ILP) \cite{mugg:ilp,ilpintro}.
As with other forms of ML, the goal of an ILP system is to learn a hypothesis that correctly generalises given training examples.
However, whereas most forms of ML represent data (examples and hypotheses) as tables, ILP represents data as logic programs.
Moreover, whereas most forms of ML learn \emph{functions}, ILP learns \emph{relations}.


Rather than refine a clause \cite{foil,progol,claudien,tilde,aleph,atom}, or a hypothesis \cite{mis,hyper,raspal,metagol}, our approach refines the \emph{hypothesis space} through learned \emph{hypothesis constraints}.
In other words, in our approach continually builds a set of constraints.
The more constraints we learn, the more we reduce the hypothesis space.
By reasoning about the hypothesis space, our approach can drastically prune large parts of the hypothesis space by testing a single hypothesis.

Atom \cite{atom} learns definite programs using SAT solvers and also learns constraints.
However, because it builds on Progol \cite{progol}, and thus employs inverse entailment, Atom struggles to learn recursive programs because it needs examples of both the base and step cases of a recursive program.
For the same reason, Atom struggles to learn optimal solutions.
By contrast, \name{} can learn recursive and optimal solutions because it learns programs rather than individual clauses.


\subsection{Recursion}
\label{sec:recursion}
Learning recursive programs has long been considered a difficult problem in ILP \cite{ilp20}.
Without recursion, it is often difficult for an ILP system to generalise from small numbers of examples \cite{crop:datacurate}.
Indeed, many popular ILP systems, such as FOIL \cite{foil}, Progol \cite{progol}, TILDE \cite{tilde}, and Aleph \cite{aleph} struggle to learn recursive programs.
The reason is that they employ a set covering approach to build a hypothesis clause by clause.
Each clause is usually found by searching an ordering over clauses.
A common approach is to pick an uncovered example, generate the bottom clause \cite{progol} for this example, the logically most specific clause that entails the example, and then to search the subsumption lattice (either top-down or bottom-up) bounded by this bottom clause.
Systems that implement this approach are often efficient because the hypothesis search is example-driven.
However, these systems tend to learn overly specific solutions and struggle to learn recursive programs \cite{hyper,ilp30}.
To overcome this limitation, \name{} searches over logic programs (sets of clauses), a technique used by other ILP systems \cite{hyper,raspal,ilasp,metagol,dilp,hexmil}.


\subsection{Optimality}
There are often multiple (sometimes infinite) hypotheses that explain the data.
Deciding which hypothesis to choose is a difficult problem.
Many ILP systems \cite{progol,aleph,tilde,xhail} are not guaranteed to learn optimal solutions, where optimal typically means the smallest program or the program with the minimal description length.
The claimed advantage of learning optimal solutions is better generalisation.
Recent meta-level ILP approaches often learn optimal solutions, such as programs with the fewest clauses \cite{mugg:metagold,metagol,hexmil} or literals \cite{aspal,ilasp}.
\name{} also learns optimal solutions, measured as the total number of literals in the hypothesis.

\subsection{Language bias}
\label{sec:langbias}
ILP approaches use a language bias \cite{ilp:book} to restrict the hypothesis space.
Language bias can be categorised as \emph{syntactic bias}, which restricts the syntax of hypotheses, such as the number of variables allowed in a clause, and \emph{semantic bias}, which restricts hypotheses based on their semantics, such as whether they are functional, irreflexive, etc.

Mode declarations \cite{progol} are a popular language bias \cite{tilde,aleph,xhail,tal,aspal,raspal,atom,ilasp}.
Mode declarations state which predicate symbols may appear in a clause, how often they may appear, the types of their arguments, and whether their arguments must be ground. We do not use mode declarations.
We instead use a simple language bias which we call \emph{predicate declarations} (Section \ref{sec:setting}), where a user needs only state whether a predicate symbol may appear in the head or/and body of a clause.
Predicate declarations are almost identical to determinations in Aleph \cite{aleph}.
The only difference is a minor syntactic one.
In addition to \emph{predicate declarations}, a user can provide other language biases, such as type information, as \emph{hypothesis constraints} (Section \ref{sec:hypothesis-constraints}).


Metarules \cite{crop:reduce} are another popular syntactic bias used by many ILP approaches \cite{raedt:clint,wang2014structure,awscp,hexmil}, including Metagol \cite{mugg:metagold,metaho,metagol} and, to an extent\footnote{\dilp{} uses program templates to essentially generate sets of metarules.}, \dilp{} ~\cite{dilp}.
A metarule is a higher-order clause which defines the exact form of clauses in the hypothesis space.
For instance, the \emph{chain} metarule is of the form $P(A,B) \leftarrow Q(A,C), R(C,B)$, where $P$, $Q$, and $R$ denote predicate variables, and allows for instantiated clauses such as \tw{last(A,B):- reverse(A,C),head(C,B)}.
Compared with predicate (and mode) declarations, metarules are a much stronger inductive bias because they specify the exact form of clauses in the hypothesis space.
However, the major problem with metarules is determining which ones to use \cite{crop:reduce}.
A user must either (i) provide a set of metarules, or (ii) use a set of metarules restricted to a certain fragment of logic, e.g.~dyadic Datalog \cite{crop:reduce}.
This limitation means that ILP systems that use metarules are difficult to use, especially when the BK contains predicate symbols with arity greater than two.
If suitable metarules are known, then, as we show in Appendix \ref{app:metarules}, \name{} can simulate metarules through hypothesis constraints.





\subsection{Answer set programming}
\label{sec:asp}


Much recent work in ILP uses ASP to learn Datalog \cite{apperception}, definite \cite{mugg:metalearn,hexmil,brute}, normal \cite{xhail,aspal,raspal}, and answer set programs \cite{ilasp}.
ASP is a declarative language that supports language features such as aggregates and weak and hard constraints.
Most ASP solvers only work on ground programs \cite{clingo}\footnote{A notable exception is Alpha Solver \cite{DBLP:conf/lpnmr/Weinzierl17}.}.
Therefore, a major limitation of most pure ASP-based ILP systems is the intrinsic grounding problem, especially on large domains, such as reasoning about lists or numbers -- most ASP implementations do not support lists nor real numbers.
For instance, ILASP \cite{ilasp} can represent real numbers as strings and delegate the reasoning to Python via Clingo's scripting feature \cite{clingo}.
However, in this approach, the numeric computation is performed when grounding the inputs, so the grounding must be finite.
Difficulty handling large (or infinite) domains is not specific to ASP.
For instance, \dilp{} uses a neural network to induce programs, but only works on BK formed of a finite set of ground atoms.
To overcome this grounding limitation, \name{} combines ASP and Prolog.
\name{} uses ASP to generate definite programs, which allows it to reason about large and infinite problem domains, such as reasoning about lists and real numbers.


ILASP3 \cite{law:thesis} is a pure ASP-based ILP system that also employs a constrain loop.
ILASP3 learns unstratified ASP programs, including programs with choice rules and weak and hard constraints, and can handle noise.
By contrast, \name{} learns Prolog programs, including programs operating over lists and real numbers, but cannot handle noise.
ILASP3 pre-computes every clause in the hypothesis space defined by a set of given mode declarations.
As we show in Experiment 1 (Section \ref{sec:buttons}), this approach struggles to learn clauses with many body literals.
By contrast, \name{} does not pre-compute every clause, which allows it to learn clauses with many body literals.
With each iteration, ILASP3 finds the best hypothesis it can.
If the hypothesis does not cover one of the examples, ILASP3 finds a reason why and then generates constraints to guide subsequent search\footnote{
This statement covers the noiseless ILASP3 setting.
Things are slightly more complicated in noisy tasks where examples are given penalties and ILASP3 may return a hypothesis that does not cover all examples, but is optimal with respect to the penalties.
Since \name{} does not yet support noise, we only consider the noiseless ILASP3 setting.
}.
The constraints are boolean formulas over the rules in the hypothesis space, an approach that requires a set of pre-computed rules.
This approach can be very expensive to compute because in the worst-case ILASP3 may need to consider every hypothesis to build a constraint (although this worst-case scenario is unlikely).
Another way of viewing ILASP3 is that it uses a counter-example guided \cite{cegis} approach and translates an uncovered example $e$ into a constraint that is satisfied if and only if $e$ is covered.
By contrast, when a hypothesis fails, \name{} translates the hypothesis itself into a set of \emph{hypothesis constraints}.
\name{}'s constraints do not reason about specific clauses (because we do not pre-compute the hypothesis space), but instead reason about the syntax of hypotheses using theta-subsumption and are therefore quick to compute.
Another subtle difference is how often the constrain loop is employed in ILASP3 and \name{}.
ILASP3's constraint loop requires at most $|E|$ iterations, where $|E|$ is the number of ILASP examples, which are partial interpretations.
Because ILASP3's examples are partial interpretations \cite{ilasp}, it is possible to represent multiple atomic examples in a single partial interpretation example.
In fact, each learning task in this paper can be represented as a single ILASP positive example \cite{ilasp}.
If represented this way, ILASP3 will generate at most one constraint (which will be satisfied if and only if a hypothesis covers the example).
For this reason, ILASP3 performs much better if the examples are split into one (partial interpretation) example per atomic example.
By contrast, the constraint loop of \name{} is not bound by the number of examples but by the size of the hypothesis space.

\subsection{Hypothesis constraints}
\label{sec:hypothesis-constraints}

Constraints are fundamental to our idea.
Many ILP systems allow a user to constrain the hypothesis space though clause constraints \cite{progol,aleph,tilde,atom,ilasp}.
For instance, Progol, Aleph, and TILDE allow for a user to provide constraints on clauses that should not be violated.
\name{} also allows a user to provide clause constraints.
\name{} additionally allows a user to provide \emph{hypothesis constraints} (or \emph{meta-constraints})\footnote{
The term \emph{hypothesis constraint} is also used in existing work \cite{DBLP:conf/ilp/SrinivasanK05,DBLP:journals/jmlr/CostaSCBDJSVL03} as an optional set of constraints on acceptable hypotheses, but without any further explanation.
}, which are constraints over a whole hypothesis (a set of clauses), not an individual clause.
As a trivial example, suppose you want to disallow two predicate symbols \emph{p/2} and \emph{q/2} from both simultaneously appearing in a program (in any body literal in any clause).
Then, because \name{} reasons at the meta-level, this restriction is trivial to express:
\[\tw{:- body\_literal(\_,p,2,\_), body\_literal(\_,q,2,\_).}\]
This constraint prunes hypotheses where the predicate symbols $p/2$ and $q/2$ both appear in the body of a hypothesis (possibly in different clauses).
The key thing to notice is the ease, uniformity, and succinctness of expressing constraints.
We introduce our full meta-level encoding in Section \ref{sec:impl}.

Declarative hypothesis constraints have many advantages.
For instance, through hypothesis constraints, \name{} can enforce (optional) type, metarule, recall, and functionality restrictions.
Moreover, hypothesis constraints allow us to prune recursive programs without a base case and subsumption redundant programs.
Finally, and most importantly, hypothesis constraints allow us to prune generalisations and specialisations of failed hypotheses, which we discuss in the next section.

Athakravi et al. \cite{DBLP:conf/ilp/AthakraviABRS14} introduce \emph{domain-dependent constraints}, which are constraints on the hypothesis space provided as input by a user.
INSPIRE \cite{inspire} also uses predefined constraints to remove redundancy from the hypothesis space (in INSPIRE's case, each hypothesis is a single clause).
\name{} also supports such constraints but goes further by learning constraints from failed hypotheses.

\section{Problem setting}
\label{sec:setting}
We now define our problem setting.


\subsection{Logic preliminaries}
We assume familiarity with logic programming notation \cite{lloyd:book} but we restate some key terminology.
All sets are finite unless otherwise stated.
A \emph{clause} is a set of literals.
A \emph{clausal theory} is a set of clauses.
A \emph{Horn} clause is a clause with at most one positive literal.
A \emph{Horn theory} is a set of Horn clauses.
A \emph{definite} clause is a Horn clause with exactly one positive literal.
A \emph{definite} theory is a set of definite clauses.
A Horn clause is a \emph{Datalog} clause if it contains no function symbols and every variable that appears in the head of the clause also appears in the body of the clause.
A \emph{Datalog} theory is a set of Datalog clauses.
Simultaneously replacing variables $v_1,\dots,v_n$ in a clause with terms $t_1,\dots,t_n$ is a \emph{substitution} and is denoted as $\theta = \{v_1/t_1,\dots,v_n/t_n\}$.
A substitution $\theta$ unifies atoms $A$ and $B$ when $A\theta = B\theta$.
We will often use \emph{program} as a synonym for \emph{theory}, e.g.~a \emph{definite program} as a synonym for a \emph{definite theory}.

\subsection{Problem setting}
Our problem setting is based on the ILP learning from entailment setting \cite{luc:book}.
Our goal is to take as input positive and negative examples of a target predicate, background knowledge (BK), and to return a hypothesis (a logic program) that with the BK entails all the positive and none of the negative examples.
In this paper, we focus on learning definite programs.
We will generalise the approach to non-monotonic programs in future work.

ILP approaches search a \emph{hypothesis space}, the set of learnable hypotheses.
ILP approaches restrict the hypothesis space through a language bias (Section \ref{sec:langbias}).
Several forms of language bias exist, such as mode declarations \cite{progol}, grammars \cite{cohen:grammarbias} and metarules \cite{crop:reduce}.
We use a simple language bias which we call \emph{predicate declarations}.
A predicate declaration simply states which predicate symbols may appear in the head (\emph{head declarations}) or body (\emph{body declarations}) of a clause in a hypothesis:

\begin{definition}[Head declaration]
\label{def:head_decl}
A \emph{head} declaration is a ground atom of the form \emph{head\_pred(p,a)} where $p$ is a predicate symbol of arity $a$.
\end{definition}

\begin{definition}[Body declaration]
\label{def:body_decl}
A \emph{body} declaration is a ground atom of the form \emph{body\_pred(p,a)} where $p$ is a predicate symbol of arity $a$.
\end{definition}

\noindent
Predicate declarations are almost identical to Aleph's \emph{determinations} \cite{aleph} but with a minor syntactical difference because determinations are of the form:
\[\emph{determination(TargetName/Arity,BackgroundName/Arity)}.\]

\noindent
A \emph{declaration bias} $D$ is a pair $(D_h,D_b)$ of sets of head ($D_h$) and body ($D_b$) declarations.
We define a \emph{declaration consistent} clause:

\begin{definition}[Declaration consistent clause]
\label{def:decl_cons_clause}
Let $D=(D_h,D_b)$ be a declaration bias and $C = h \leftarrow b_1,b_2,\dots,b_n$ be a definite clause.
Then $C$ is \emph{declaration consistent} with $D$ if and only if:

\begin{itemize}
    \setlength\itemsep{0pt}
    \setlength\parskip{0pt}
    \item $h$ is an atom of the form $p(X_1,\dots,X_n)$ and \emph{head\_pred(p,n)} is in $D_h$
    \item every $b_i$ is a literal of the form $p(X_1,\dots,X_n)$ and \emph{body\_pred($p,n$)} is in $D_b$
    \item every $X_i$ is a first-order variable
\end{itemize}
\end{definition}

\begin{example}[Declaration consistency]
Let $D$ be the declaration bias:
\[(\{\emph{head\_pred(targ,2)}\}, \{ \emph{body\_pred(head,2)},\emph{body\_pred(tail,2)}\})\]
Then the following clauses are all consistent with $D$:
\[
\begin{array}{l}
    \tw{targ(A,B):- head(A,C).}\\
    \tw{targ(A,A):- head(B,A).}\\
    \tw{targ(A,B):- head(A,C),tail(C,B).}
  \end{array}
\]

\noindent
By contrast, the following clauses are inconsistent with $D$:
\[
\begin{array}{l}
    \tw{targ(A):- head(A,C).}\\
    \tw{targ(A,B):- targ(A,B).}\\
    \tw{tail(A,B):- reverse(A,C),tail(C,B).}
\end{array}
\]
\end{example}

\noindent
We define a \emph{declaration consistent} hypothesis:

\begin{definition}[Declaration consistent hypothesis]
\label{def:decl_cons_hypothesis}
A \emph{declaration consistent} hypothesis $H$ is a set of definite clauses where each $C \in H$ is declaration consistent with $D$.
\end{definition}

\begin{example}[Declaration consistent hypothesis]
Let $D$ be the declaration bias:
\[(\{\emph{head\_pred(targ,2)}\}, \{ \emph{body\_pred(head,2)},\emph{body\_pred(tail,2)}\})\]
Then two declaration consistent hypotheses are:
\[
    \begin{array}{l}
    \tw{h$_1$} : \left\{
    \begin{array}{l}
        \tw{targ(A,B):- head(A,B)}
    \end{array}
    \right\}\\
    \tw{h$_2$} : \left\{
    \begin{array}{l}
        \tw{targ(A,B):- head(A,B).}\\
        \tw{targ(A,B):- tail(A,C),head(C,B).}
    \end{array}
    \right\}
    \end{array}
  \]
\end{example}

\noindent
In addition to a declaration bias, we restrict the hypothesis space through \emph{hypothesis constraints}.
We first clarify what we mean by a \emph{constraint}:

\begin{definition}[Constraint]
A \emph{constraint} is a Horn clause without a head, i.e.~a \emph{denial}.
We say that a constraint is \emph{violated} if all of its body literals are true.
\end{definition}

\noindent
Rather than define hypothesis constraints for a specific encoding (e.g.~the encoding we use in Section \ref{sec:impl}), we use a more general definition:

\begin{definition}[Hypothesis constraint]
\label{def:hconstraint}
Let $\mathcal{L}$ be a language that defines hypotheses, i.e.~a meta-language.
Then a hypothesis constraint is a constraint expressed in $\mathcal{L}$.
\end{definition}

\begin{example}
In Section \ref{sec:impl}, we introduce a meta-language for definite programs.
In our encoding, the atom \tw{head\_literal(Clause,Pred,Arity,Vars)} denotes that the clause \tw{Clause} has a head literal with the predicate symbol \tw{Pred}, is of arity \tw{Arity}, and has the arguments \tw{Vars}.
An example hypothesis constraint in this language is:
\[\tw{:- head\_literal(\_,p,2,\_).}\]
This constraint states that a predicate symbol \tw{p} of arity 2 cannot appear in the head of any clause in a hypothesis.
\end{example}

\begin{example}
In our encoding, the atom \tw{body\_literal(Clause,Pred,Arity,Vars)} denotes that the clause \tw{Clause} has a body literal with the predicate symbol \tw{Pred}, is of arity \tw{Arity}, and has the arguments \tw{Vars}.
An example hypothesis constraint in this language is:
\[\tw{:- head\_literal(\_,p,2,\_), body\_literal(\_,p,2,\_).}\]
This constraint states that the predicate symbol \tw{p} cannot appear in the body of a clause if it appears in the head of a clause (not necessarily the same clause).
\end{example}

\noindent
We define a \emph{constraint consistent hypothesis}:

\begin{definition}[Constraint consistent hypothesis]
Let $C$ be a set of hypothesis constraints written in a language $\mathcal{L}$.
A set of definite clauses $H$ is \emph{consistent} with $C$ if, when written in $\mathcal{L}$, $H$ does not violate any constraint in $C$.
\end{definition}

\noindent
We now define our hypothesis space:

\begin{definition}[Hypothesis space]
\label{def:hspace}
Let $D$ be a declaration bias and $C$ be a set of hypothesis constraints.
Then the hypothesis space $\mathcal{H}_{D,C}$ is the set of all declaration and constraint consistent hypotheses.
We refer to any element in $\mathcal{H}_{D,C}$ as a \emph{hypothesis}.
\end{definition}

\noindent
We define the LFF problem input:

\begin{definition}[LFF problem input]
\label{def:probin}
Our problem input is a tuple $(B, D, C, E^+, E^-)$ where
\begin{itemize}
    \setlength\itemsep{0pt}
    \setlength\parskip{0pt}
    \item $B$ is a Horn program denoting background knowledge
    \item $D$ is a declaration bias
    \item $C$ is a set of hypothesis constraints
    \item $E^+$ is a set of ground atoms denoting positive examples
    \item $E^-$ is a set of ground atoms denoting negative examples
\end{itemize}
\end{definition}

\noindent
Note that $C$, $E^+$, and $E^-$ can be empty sets (but $E^+$ and $E^-$ cannot both be empty).
We assume that no predicate symbol in the body of a clause in $B$ appears in a head declaration of $D$.
In other words, we assume that the BK does not depend on any hypothesis.

For convenience, we define different types of hypotheses, mostly using standard ILP terminology \cite{ilp:book}:

\begin{definition}[Hypothesis types]
\label{def:htypses}
Let $(B, D, C, E^+, E^-)$ be an input tuple and $H \in \mathcal{H}_{D,C}$ be a hypothesis.
Then $H$ is:
\begin{itemize}
    \setlength\itemsep{0pt}
    \setlength\parskip{0pt}
    \item \emph{Complete} when $\forall e \in E^+ \; H \cup B \models e$
    \item \emph{Consistent} when $\forall e \in E^-, \; H \cup B \not\models e$
    \item \emph{Incomplete} when $\exists e \in E^+, \; H \cup B \not \models e$
    \item \emph{Inconsistent} when $\exists e \in E^-, \; H \cup B \models e$
    \item \emph{Totally incomplete} when $\forall e \in E^+, \; H \cup B \not \models e$
\end{itemize}
\end{definition}

\noindent
We define a LFF \emph{solution}, i.e.~our problem output:

\begin{definition}[LFF solution]
\label{def:solution}
Given an input tuple $(B, D, C, E^+, E^-)$, a hypothesis $H \in \mathcal{H}_{D,C}$ is a \emph{solution} when $H$ is complete and consistent.
\end{definition}

\noindent
Conversely, we define a \emph{failed hypothesis}:

\begin{definition}[Failed hypothesis]
\label{def:failed}
Given an input tuple $(B, D, C, E^+, E^-)$, a hypothesis $H \in \mathcal{H}_{D,C}$ \emph{fails} (or is a \emph{failed} hypothesis) when $H$ is either incomplete or inconsistent.
\end{definition}

\noindent
There may be multiple (sometimes infinite) solutions.
We want to find the smallest solution:

\begin{definition}[Hypothesis size]
\label{def:size}
The function $size(H)$ returns the total number of literals in the hypothesis $H$.
\end{definition}

\noindent
We define an \emph{optimal solution}:

\begin{definition}[Optimal solution]
\label{def:opthyp}
Given an input tuple $(B, D, C, E^+, E^-)$, a hypothesis $H \in \mathcal{H}_{D,C}$ is an \emph{optimal solution} when two conditions hold:
\begin{itemize}
    \setlength\itemsep{0pt}
    \setlength\parskip{0pt}
    \item $H$ is a solution
    \item $\forall H' \in \mathcal{H}_{D,C}$, such that $H'$ is a solution, $size(H) \leq size(H')$
\end{itemize}
\end{definition}


\subsection{Hypothesis space}
The purpose of LFF is to reduce the size of the hypothesis space through learned hypothesis constraints.
The size of the unconstrained hypothesis space is a function of a declaration bias and additional bounding variables:

\begin{proposition}[Hypothesis space size]
\label{prop:hspace}
Let $D=(D_h,D_b)$ be a declaration bias with a maximum arity $a$, $v$ be the maximum number of unique variables allowed in a clause, $m$ be the maximum number of body literals allowed in a clause, and $n$ be the maximum number of clauses allowed in a hypothesis.
Then the maximum number of hypotheses in the unconstrained hypothesis space is:
\[\sum_{j=1}^n {|D_h|v^a \; \sum_{i=1}^m {|D_b|v^a \choose i} \choose j} \]
\end{proposition}
\begin{proof}
Let $C$ be an arbitrary clause in the hypothesis space.
There are $|D_h|v^a$ ways to define the head literal of $C$.
There are $|D_b|v^a$ ways to define a body literal in $C$.
The body of $C$ is a set of literals.
There are $|D_b|v^a \choose k$ ways to choose $k$ body literals.
We bound the number of body literals to $m$, so there are $\sum_{i=1}^m {|D_b|v^a \choose i}$ ways to choose at most $m$ body literals.
Therefore, there are $|D_h|v^a \; \sum_{i=1}^m {|D_b|v^a \choose i}$ ways to define $C$.
A hypothesis is a set of definite clauses.
Given $n$ clauses, there are $n \choose k$ ways to choose $k$ clauses to form a hypothesis.
Therefore, there are $\sum_{j=1}^n {|D_h|v^a \; \sum_{i=1}^m {|D_b|v^a \choose i} \choose j}$ ways to define a hypothesis with at most $n$ clauses.
\end{proof}

\noindent
As this result shows, the hypothesis space is huge for non-trivial inputs, which motivates using learned constraints to prune the hypothesis space.

\subsection{Generalisations and specialisations}
\label{sec:genandspec}
To prune the hypothesis space, we learn constraints to remove \emph{generalisations} and \emph{specialisations} of failed hypotheses.
We reason about the generality of hypotheses syntactically through $\theta$-subsumption (or \emph{subsumption} for short) \cite{plotkin:thesis}:

\begin{definition}[Clausal subsumption]
\label{def:clausesub}
A clause $C_1$ \emph{subsumes} a clause $C_2$ if and only if there exists a substitution $\theta$ such that $C_1\theta \subseteq C_2$.
\end{definition}

\begin{example}[Clausal subsumption]
Let \tw{C$_1$} and \tw{C$_2$} be the clauses:
\[
    \begin{array}{l}
    \tw{C$_1$ = f(A,B):- head(A,B)}\\
    \tw{C$_2$ = f(X,Y):- head(X,Y),odd(Y)}.
    \end{array}
\]
Then $\tw{C}_1$ subsumes $\tw{C}_2$ because $\tw{C}_1\theta \subseteq \tw{C}_2$ with $\theta = \{A/X,Y/B\}$.
\end{example}

\noindent
If a clause $C_1$ subsumes a clause $C_2$ then $C_1$ entails $C_2$ \cite{ilp:book}.
However, if $C_1$ entails $C_2$ then it does not necessarily follow that $C_1$ subsumes $C_2$.
Subsumption is therefore weaker than entailment.
However, whereas checking entailment between clauses is undecidable \cite{church:problem}, checking subsumption between clauses is decidable, although, in general, deciding subsumption is a NP-complete problem \cite{ilp:book}.

Midelfart \cite{midelfart} extends subsumption to clausal theories:
\begin{definition}[Theory subsumption]
\label{def:theorysub}
A clausal theory $T_1$ subsumes a clausal theory $T_2$, denoted $T_1 \preceq T_2$, if and only if $\forall C_2 \in T_2, \exists C_1 \in T_1$ such that $C_1$ subsumes $C_2$.
\end{definition}

\begin{example}[Theory subsumption]
Let \tw{h$_1$}, \tw{h$_2$}, and \tw{h$_3$} be the clausal theories:
\[
\begin{array}{l}
\tw{h$_1$} = \left\{
    \begin{array}{l}
        \tw{f(A,B):- head(A,B).}\\
    \end{array}
    \right\}\\
\tw{h$_2$} = \left\{
    \begin{array}{l}
        \tw{f(A,B):- head(A,B),odd(B).}\\
    \end{array}
    \right\}\\
\tw{h$_3$} = \left\{
    \begin{array}{l}
        \tw{f(A,B):- head(A,B).}\\
        \tw{f(A,B):- reverse(A,C),head(C,B).}\\
    \end{array}
    \right\}
\end{array}
\]
\noindent
Then \tw{h$_1$} $\preceq$ \tw{h$_2$}, \tw{h$_3$} $\preceq$ \tw{h$_1$}, and \tw{h$_3$} $\preceq$ \tw{h$_2$}.
\end{example}

\noindent
Theory subsumption also implies entailment:
\begin{proposition}[Subsumption implies entailment]
\label{prop:ourorder}
Let $T_1$ and $T_2$ be clausal theories.
If $T_1 \preceq T_2$ then $T_1 \models T_2$.
\end{proposition}
\begin{proof}
Follows trivially from the definitions of clausal subsumption (Definition \ref{def:clausesub}) and theory subsumption (Definition \ref{def:theorysub}).
\end{proof}

\noindent
We use theory subsumption to define a \emph{generalisation}:
\begin{definition}[Generalisation]
\label{def:generalisation}
A clausal theory $T_1$ is a \emph{generalisation} of a clausal theory $T_2$ if and only if $T_1 \preceq T_2$.
\end{definition}

\noindent
We likewise define our notion of a \emph{specialisation}:

\begin{definition}[Specialisation]
\label{def:specialisation}
A clausal theory $T_1$ is a \emph{specialisation} of a clausal theory $T_2$ if and only if $T_2 \preceq T_1$.
\end{definition}

\noindent
In the next section, we use these definitions to define constraints to prune the hypothesis space.

\subsection{Learning constraints from failures}
\label{sec:failures}


In the test stage of LFF, a learner tests a hypothesis against the examples.
A hypothesis fails when it is incomplete or inconsistent.
If a hypothesis fails, a learner learns hypothesis constraints from the different types of \emph{failures}.
We define two general types of constraints, \emph{generalisation} and \emph{specialisation}, which apply to any clausal theory, and show that they are sound in that they do not prune solutions.
We also define an \emph{elimination} constraint, which, under certain assumptions, allows us to prune programs that generalisation and specialisation constraints do not, and which we show is sound in that it does not prune optimal solutions.
We describe these constraints in turn.

\subsubsection{Generalisations and specialisations}
To illustrate generalisations and specialisations, suppose we have positive examples $E^+$, negative examples $E^-$, background knowledge $B$, and a hypothesis $H$.
First consider the outcomes of testing $H$ against $E^-$:
\begin{center}
\small
\begin{tabular}{l|l|l}
\toprule
\textbf{Outcome} & \textbf{Description} & \textbf{Formula}
\\
\midrule
\nnone{}
&
$H$ is consistent, i.e.~$H$ entails no negative example & $\forall e \in E^-, \; H \cup B \not\models e$\\
\midrule
\nsome{}
&
$H$ is inconsistent, i.e.~$H$ entails at least one negative example & $\exists e \in E^-, \; H \cup B \models e$\\
\bottomrule
\end{tabular}
\end{center}

\noindent
Suppose the outcome is \nnone{}, i.e.~$H$ is consistent.
Then we cannot prune the hypothesis space.

Suppose the outcome is \nsome{}, i.e.~$H$ is inconsistent.
Then $H$ is too general so we can prune generalisations (Definition \ref{def:generalisation}) of $H$.
A constraint that only prunes generalisations is a \emph{generalisation} constraint:

\begin{definition}[Generalisation constraint]
\label{def:generalisation_constraint}
A generalisation constraint only prunes generalisations of a hypothesis from the hypothesis space.
\end{definition}

\begin{example}[Generalisation constraint]
\label{ex:gen-constraint}
Suppose we have the negative examples $E^-$ and the hypothesis \tw{h}:
\[
E^- = \left\{
\begin{array}{l}
\tw{last([a,n,n],a)}\\
\end{array}
\right\}
\hspace{3ex}
\tw{h} = \left\{
\begin{array}{l}
\tw{last(A,B):- head(A,B)}.
\end{array}
\right\}
\]
\noindent
Because \tw{h} entails a negative example, it is too general, so we can prune generalisations of it, such as \tw{h$_1$} and \tw{h$_2$}:
\[
\begin{array}{l}
\tw{h$_1$} = \left\{
    \begin{array}{l}
        \tw{last(A,B):-head(A,B).}\\
        \tw{last(A,B):-tail(A,C),head(C,B).}
    \end{array}
    \right\}\\
\tw{h$_2$} = \left\{
    \begin{array}{l}
        \tw{last(A,B):-head(A,B).}\\
        \tw{last(A,B):-tail(A,C),head(C,B),head(A,B).}
    \end{array}
    \right\}
    \end{array}
\]
\end{example}

\noindent
We show that pruning generalisations of an inconsistent hypothesis is \emph{sound} in that it only prunes inconsistent hypotheses, i.e. does not prune consistent hypotheses:
\begin{proposition}[Generalisation soundness]
\label{prop:generalisation_soundness}
Let $(B, D, C, E^+, E^-)$ be a problem input, $H \in \mathcal{H}_{D,C}$ be an inconsistent hypothesis, and $H' \in \mathcal{H}_{D,C}$ be a hypothesis such that $H' \preceq{} H$.
Then $H'$ is inconsistent.
\end{proposition}
\begin{proof}
Follows from Proposition \ref{prop:ourorder}.
\end{proof}

\noindent
Now consider the outcomes\footnote{The outcomes are not mutually exclusive.} of testing $H$ against $E^+$:

\begin{center}
\small
\begin{tabular}{l|l|l}
\toprule
\textbf{Outcome} & \textbf{Description} & \textbf{Formula}\\
\midrule
\pall{}
&
$H$ is complete, i.e.~$H$ entails all positive examples & $\forall e \in E^+, \; H \cup B \models e$\\
\midrule
\psome{}
&
$H$ is incomplete, i.e. $H$ does entail all positive examples & $\exists e \in E^+, \; H \cup B \not\models e$\\
\midrule
\pnone{}
&
$H$ is totally incomplete, i.e.~$H$ entails no positive examples & $\forall e \in E^+, \; H \cup B \not\models e$\\
\bottomrule
\end{tabular}
\end{center}


\noindent
Suppose the outcome is \pall{}, i.e.~$H$ is complete.
Then we cannot prune the hypothesis space.

Suppose the outcome is \psome{}, i.e.~is incomplete.
Then $H$ is too specific so we can prune specialisations (Definition \ref{def:specialisation}) of $H$.
A constraint that only prunes specialisations of a hypothesis is a \emph{specialisation constraint}:
\begin{definition}[Specialisation constraint]
\label{def:specialisation_constraint}
A specialisation constraint only prunes specialisations of a hypothesis from the hypothesis space.
\end{definition}

\begin{example}[Specialisation constraint]
\label{ex:spec-constr}
Suppose we have the positive examples $E^+$ and the hypothesis \tw{h}:
\[
E^+ = \left\{
\begin{array}{l}
\tw{last([b,o,b],b)}\\
\tw{last([a,l,i,c,e],e)}\\
\end{array}
\right\}
\hspace{3ex}
\tw{h} = \left\{
\begin{array}{l}
\tw{last(A,B):- head(A,B)}.
\end{array}
\right\}
\]

\noindent
Because \tw{h} entails the first example but not the second it is too specific.
We can therefore prune specialisations of \tw{h}, such as \tw{h$_1$} and \tw{h$_2$}:

\[
\begin{array}{l}
\tw{h$_1$} = \left\{
    \begin{array}{l}
        \tw{last(A,B):- head(A,B),empty(A).}\\
    \end{array}
    \right\}\\
\tw{h$_2$} = \left\{
    \begin{array}{l}
        \tw{last(A,B):- head(A,B),tail(A,C).}\\
    \end{array}
    \right\}
    \end{array}
\]
\end{example}

\noindent
We show that pruning specialisations of an incomplete hypothesis is \emph{sound} because it only prunes incomplete hypotheses, i.e. does not prune complete hypotheses:
\begin{proposition}[Specialisation soundness]
\label{prop:specialisation_soundness}
Let $(B, D, C, E^+, E^-)$ be a problem input, $H \in \mathcal{H}_{D,C}$ be an incomplete hypothesis, and $H' \in \mathcal{H}_{D,C}$ be a hypothesis such that $H \preceq{} H'$. Then $H'$ is incomplete.
\end{proposition}
\begin{proof}
Follows from Proposition \ref{prop:ourorder}.
\end{proof}

\subsubsection{Eliminations}
\label{sec:eliminations}

\noindent
Suppose the outcome is \pnone{}, i.e.~$H$ is totally incomplete.
Then $H$ is too specific so, as with \psome{}, we can prune specialisations of $H$.
However, because $H$ is totally incomplete (i.e does not entail \emph{any} positive example), under certain assumptions, we can prune more.
If $H$ is totally incomplete then there is no need for $H$ to appear in a complete and \emph{separable} hypothesis:
\begin{definition}[Separable]
A \emph{separable} hypothesis $G$ is one where no predicate symbol in the head of a clause in $G$ occurs in the body of clause in $G$.
\end{definition}
\noindent
Note that separable programs include recursive programs.
\begin{example}[Non-separable hypothesis]
The following hypothesis is non-separable because \tw{f1/2} appears in the head and body of the program:
\[
\left\{
\begin{array}{l}
\tw{f(A,B):- f1(A,C),head(C,B).}\\
\tw{f1(A,B):- tail(A,C),tail(C,B).}
\end{array}
\right\}
\]
\noindent
The following hypothesis is non-separable because \tw{last/2} appears in the head and body of the program:
\[
\left\{
\begin{array}{l}
\tw{last(A,B):- head(A,B),tail(A,C),empty(C).}\\
\tw{last(A,B):- tail(A,C),last(C,B).}
\end{array}
\right\}
\]

\end{example}
\noindent
In other words, if $H$ is totally incomplete and does not entail \emph{any} positive example, then no specialisation of $H$ can appear in an optimal separable solution.
We can therefore prune separable hypotheses that contain specialisations of $H$.
We call such a constraint an \emph{elimination constraint}:

\begin{definition}[Elimination constraint]
\label{def:elimination_constraint}
An elimination constraint only prunes separable hypotheses that contain specialisations of a hypothesis from the hypothesis space.
\end{definition}

\begin{example}[Elimination constraint]
\label{ex:elim-constraint}
Suppose we have the positive examples $E^+$ and the hypothesis \tw{h}:
\[
E^+ = \left\{
\begin{array}{l}
\tw{last([b,o,b],b)}\\
\tw{last([a,l,i,c,e],e)}\\
\end{array}
\right\}
\hspace{3ex}
\tw{h} = \left\{
\begin{array}{l}
\tw{last(A,B):- tail(A,C),head(C,B).}
\end{array}
\right\}
\]

\noindent
Because \tw{h} does not entail any positive example there is no reason for \tw{h} (nor its specialisations) to appear in a separable hypothesis.
We can therefore prune separable hypotheses which contain specialisations of \tw{h}, such as:
\[
\begin{array}{l}
\tw{h$_1$} = \left\{
    \begin{array}{l}
        \tw{last(A,B):-head(A,B).}\\
        \tw{last(A,B):-tail(A,C),head(C,B).}\\
    \end{array}
    \right\}\\
\tw{h$_2$} = \left\{
    \begin{array}{l}
        \tw{last(A,B):-head(A,B).}\\
        \tw{last(A,B):-tail(A,C),head(C,B),odd(B).}\\
    \end{array}
    \right\}
\\
\tw{h$_3$} = \left\{
    \begin{array}{l}
        \tw{last(A,B):-head(A,B),even(B).}\\
        \tw{last(A,B):-tail(A,C),head(C,B),odd(B).}\\
    \end{array}
    \right\}
    \end{array}
\]
\end{example}

\noindent
Elimination constraints are not sound in the same way as the generalisation and specialisation constraints because they prune solutions (Definition \ref{def:solution}) from the hypothesis space.

\begin{example}[Elimination solution unsoundness]
Suppose we have the positive examples $E^+$ and the hypothesis \tw{h$_1$}:
\[
E^+ = \left\{
\begin{array}{l}
\tw{last([j,i,m],m)}\\
\tw{last([a,l,i,c,e],e)}\\
\end{array}\\
\right\}\\
\hspace{3ex}
\tw{h$_1$} = \left\{
\begin{array}{l}
\tw{last(A,B):- head(A,B).}
\end{array}
\right\}
\]
Then an elimination constraint would prune the complete hypothesis \tw{h$_2$}:
\[
\tw{h$_2$} = \left\{
\begin{array}{l}
\tw{last(A,B):- head(A,B).}\\
\tw{last(A,B):- reverse(A,C),head(C,B).}
\end{array}
\right\}
\]
\end{example}
\noindent
However, for separable definite programs, elimination constraints are \emph{sound} with respect to optimal solutions, i.e.~they only prune non-optimal solutions from the hypothesis space.
To show this result, we first introduce a lemma:

\begin{lemma}
\label{lemma1}
Let $(B, D, C, E^+, E^-)$ be a problem input, $D = (D_h, D_b)$ be head and body declarations, $H_1 \in \mathcal{H}_{D,C}$ be a totally incomplete hypothesis, $H_2 \in \mathcal{H}_{D,C}$ be a complete and separable hypothesis such that $H_1 \subset{} H_2$, and $H_3 = H_2 \setminus H_1$.
Then $H_3$ is complete.
\end{lemma}
\begin{proof}
By assumption, no predicate symbol in $D_h$ occurs in the body of a clause in $B$, $H_2$ (since $H_2$ is separable), nor $H_1$ (since $H_1 \subset H_2$), i.e.~no clause in a hypothesis depends on another, so we can reason about entailment using single clauses.
Since $H_1$ is totally incomplete, it holds that $\forall e \in E^+, \neg\exists C \in H_1, \{C\} \cup B \models e$.
Since $H_2$ is complete, it holds that $\forall e \in E^+, \exists C \in H_2, \{C\} \cup B \models e$.
Therefore, it is clear that $\forall e \in E^+, \exists C \in H_2, C \not\in H_1,  \{C\} \cup B \models e$, which implies $\forall e \in E^+, H_2 \setminus H_1 \cup B \models e$, and thus $H_3$ is complete.
\end{proof}

\noindent
We use this result to show that elimination constraints are \emph{sound} with respect to optimal solutions:

\begin{proposition}[Elimination optimal soundness]
\label{prop:elim_sound}
Let $(B, D, C, E^+, E^-)$ be a problem input, $D = (D_h, D_b)$ be head and body declarations, $H_1 \in \mathcal{H}_{D,C}$ be a totally incomplete hypothesis, $H_2 \in \mathcal{H}_{D,C}$ be a hypothesis such that $H_1 \preceq{} H_2$, and $H_3 \in \mathcal{H}_{D,C}$ be a separable hypothesis such that $H_2 \subset{} H_3$.
Then $H_3$ is not an optimal solution.
\end{proposition}
\begin{proof}
Assume that $H_3$ is an optimal solution.
This assumption implies that (i) $H_3$ is a solution, and (ii) there is no hypothesis $H_4 \in \mathcal{H}_{D,C}$ such that $H_4$ is a solution and $size(H_4) < size(H_3)$.
Let $H_4 = H_3 \setminus H_2$.
Since $H_1$ is totally incomplete and $H_1 \preceq{} H_2$ then, by Proposition \ref{prop:ourorder}, $H_2$ is totally incomplete.
By assumption, $H_3$ is complete and since $H_4 = H_3 \setminus H_2$ and $H_2$ is totally incomplete then, by Lemma \ref{lemma1}, $H_4$ is complete.
Because $H_3$ is consistent, then, by the monotonicity of definite programs, $H_4$ is consistent (i.e removing clauses can only make a definite program more specific).
Therefore, $H_4$ is complete and consistent and is a solution.
Since $H_4 = H_3 \setminus H_2$ and $H_2 \subset{} H_3$, then $size(H_4) < size(H_3)$.
Therefore, condition (ii) cannot hold, which contradicts the assumption and completes the proof.
\end{proof}

\noindent
This proof relies on a hypothesis $H$ being (i) a definite program and (ii) separable.
Condition (i) is clear because the proof relies on the monotonicity of definite programs.
To illustrate condition (ii), we give a counter-example to show why we cannot use elimination constraints to prune non-separable hypotheses:

\begin{example}[Non-elimination for non-separable hypotheses]
Suppose we have the positive examples $E^+$ and the hypothesis \tw{h}:
\[
\begin{array}{l}
E^+ = \left\{
\begin{array}{l}
\tw{last([a,l,a,n],n)}\\
\tw{last([t,u,r,i,n,g],g)}\\
\end{array}
\right\}
\\
\\
\tw{h} = \left\{
\begin{array}{l}
\tw{last(A,B):- head(A,B),tail(A,C),empty(C)}.
\end{array}
\right\}
\end{array}
\]
\noindent
Then \tw{h} is totally incomplete so there is no reason for \tw{h} to appear in a separable hypothesis.
However, \tw{h} can still appear in a recursive hypothesis, where the clauses depend on each other, such as \tw{h$_2$}:
\[
\tw{h$_2$} = \left\{
\begin{array}{l}
        \tw{last(A,B):- head(A,B),tail(A,C),empty(C)}.\\
        \tw{last(A,B):- tail(A,C),last(C,B)}.
\end{array}
\right\}
\]
\end{example}

\subsubsection{Constraints summary}

\noindent
To summarise, combinations of these different outcomes imply different combinations of constraints, shown in Table \ref{tab:outcomes-non-rec}.
In the next section we introduce \name{}, which uses these constraints to learn definite programs.

\begin{table}[ht]
\center
\begin{tabular}{l|l|l}
\toprule
\textbf{Outcome}             & \nnone{} & \nsome{}\\
\midrule
\pall{} & n/a & Generalisation\\
\psome{} & Specialisation & Specialisation, Generalisation\\
\pnone{} & Specialisation, Elimination & Specialisation, Elimination, Generalisation\\
\bottomrule
\end{tabular}
\caption{
    The constraints we can learn from testing a hypothesis.
    The \pall{} and \nnone{} outcomes denote that we have found a solution.
}
\label{tab:outcomes-non-rec}
\end{table}
\section{\name{}}
\label{sec:impl}


\name{} implements the LFF approach and works in three separate stages: generate, test, and constrain.
Algorithm \ref{alg:popper} sketches the \name{} algorithm which combines the three stages.
To learn optimal solutions (Definition \ref{def:opthyp}), \name{} searches for programs of increasing size.
We describe the generate, test, and constrain stages in detail, how we use ASP's multi-shot solving \cite{multishot} to maintain state between the three stages, and then prove the soundness and completeness of \name{}.

\begin{algorithm}[ht]
{
\small
\begin{myalgorithm}[]
def $\text{popper}$(e$^+$, e$^-$, bk, declarations, constraints, max_vars, max_literals, max_clauses):
  num_literals = 1
  while num_literals $\leq$ max_literals:
    program = generate(declarations, constraints, max_vars, num_literals, max_clauses)
    if program == 'space_exhausted':
      num_literals += 1
      continue
    outcome = test(e$^+$, e$^-$, bk, program)
    if outcome == ('all_positive', 'none_negative')
      return program
    constraints += learn_constraints(program, outcome)
  return {}
\end{myalgorithm}
\caption{
\name{}
}
\label{alg:popper}
}
\end{algorithm}

\label{sec:generate}

The generate step of \name{} takes as input (i) predicate declarations, (ii) hypothesis constraints, and (iii) bounds on the maximum number of variables, literals, and clauses in a hypothesis, and returns an answer set which represents a definite program, if one exists.
The idea is to define an ASP problem where an answer set (a model) corresponds to a definite program, an approach also employed by other recent ILP approaches  \cite{aspal,ilasp,hexmil,inspire}.
In other words, we define a meta-language in ASP to represent definite programs.
\name{} uses ASP constraints to ensure that a definite program is declaration consistent and obeys hypothesis constraints, such as enforcing type restrictions or disallowing mutual recursion.
By later adding learned hypothesis constraints, we eliminate answer sets, and thus reduce the hypothesis space.
In other words, the more constraints we learn, the more we reduce the hypothesis space.

Figure \ref{fig:alan} shows the base ASP program to generate programs.
The idea is to find an answer set with suitable head and body literals, which both have the arguments \tw{(Clause,Pred,Arity,Vars)} to denote that there is a literal in the clause \tw{Clause}, with the predicate symbol \tw{Pred}, arity \tw{Arity}, and variables \tw{Vars}.
For instance, \tw{head\_literal(0,p,2,(0,1))} denotes that clause \tw{0} has a head literal with the predicate symbol \tw{p}, arity \tw{2}, and variables \tw{(0,1)}, which we interpret as \tw{(A,B)}.
Likewise, \tw{body\_literal(1,q,3,(0,0,2))} denotes that clause \tw{1} has a body literal with the predicate symbol \tw{q}, arity \tw{3}, and variables \tw{(0,0,2)}, which we interpret as \tw{(A,A,C)}.
Head and body literals are restricted by \tw{head\_pred} and \tw{body\_pred} declarations respectively.
Table \ref{tab:answerset_to_prolog} shows examples of the correspondence between an answer set and a definite program, which we represent as a Prolog program.

\begin{figure}[ht]
\centering
\begin{minipage}{.8\linewidth}
\begin{lstlisting}[frame=single]
% possible clauses
allowed_clause(0..N-1):- max_clauses(N).

% variables
var(0..N-1):- max_vars(N).

% clauses with a head literal
clause(Clause):- head_literal(Clause,_,_,_).

%% head literals
0 {head_literal(Clause,P,A,Vars): head_pred(P,A), vars(A,Vars)} 1:-
    allowed_clause(Clause).

%% body literals
1 {body_literal(Clause,P,A,Vars): body_pred(P,A), vars(A,Vars)} N:-
    clause(Clause), max_body(N).

% variable combinations
vars(1,(Var1,)):- var(Var1).
vars(2,(Var1,Var2)):- var(Var1),var(Var2).
vars(3,(Var1,Var2,Var3)):- var(Var1),var(Var2),var(Var3).
\end{lstlisting}
\end{minipage}
\caption{
\name{} base ASP program.
The \tw{head\_literal} literals are bounded from 0 to 1, i.e for each possible clause there can be at most 1 head literal.
The \tw{body\_literal} literals are bounded from 1 to $N$, where $N$ is the maximum number of literals allowed in a clause, i.e. for each clause with a head literal, there has to be at least 1 but at most $N$ body literals.
}
\label{fig:alan}
\end{figure}







\begin{table}[ht]
\centering
\footnotesize
\begin{tabular}{@{}ll@{}}
\toprule
\textbf{Answer set} & \textbf{Prolog program}
\\
\midrule
\begin{tabular}[c]{@{}l@{}}
\{head\_literal(0,f,2,(0,1)),body\_literal(0,empty,(1,))\}
\end{tabular} & \tw{f(A,B):-empty(B).}
\\
\midrule
\begin{tabular}[c]{@{}l@{}}
\{head\_literal(0,f,2,(0,1)),body\_literal(0,head,2,(1,0))\}
\end{tabular} & \tw{f(A,B):-head(B,A).}
\\
\midrule
\begin{tabular}[c]{@{}l@{}}
\{head\_literal(0,f,2,(0,1)),body\_literal(0,tail,2,(0,1)),\\
body\_literal(0,tail,2,(0,2))\}
\end{tabular} & \tw{f(A,B):-tail(A,B),tail(A,C).}
\\
\midrule
\begin{tabular}[c]{@{}l@{}}
\{head\_literal(0,connected,2,(0,1)),body\_literal(0,edge,2,(0,1)),\\
head\_literal(1,connected,2,(0,1)),body\_literal(1,edge,2,(0,2)),\\
body\_literal(1,connected,(2,1))\}
\end{tabular} &

\begin{tabular}[c]{@{}l@{}}
\tw{connected(A,B):-edge(A,B).}
\\
\tw{connected(A,B):-edge(A,C),connected(C,B).}
\end{tabular}\\
\midrule
\begin{tabular}[c]{@{}l@{}}
\{head\_literal(0,last,2,(0,1)),body\_literal(0,tail,2,(0,2)),\\
body\_literal(0,empty,1,(2,)),body\_literal(0,head,2,(0,1)),\\
head\_literal(1,last,2,(0,1)),body\_literal(1,tail,2,(0,2)),\\
body\_literal(1,last,2,(2,1))\}
\end{tabular} &

\begin{tabular}[c]{@{}l@{}}
\tw{last(A,B):-tail(A,C),empty(C),head(A,B).}
\\
\tw{last(A,B):-tail(A,C),last(C,B).}
\end{tabular}\\
\bottomrule
\end{tabular}
\caption{
  The correspondence between an answer set and a definite program represented as a Prolog program.
}
\label{tab:answerset_to_prolog}
\end{table}

\subsubsection{Validity, redundancy, and efficiency constraints}
\name{} uses hypothesis constraints (in the form of ASP constraints) to eliminate answer sets, i.e. to prune the hypothesis space.
\name{} uses constraints to prune invalid programs.
For instance, Figure \ref{fig:recursion} shows constraints specifically for recursive programs, such as preventing recursion without a base case.
\name{} also uses constraints to reduce redundancy.
For instance, \name{} prunes subsumption redundant programs, such as pruning the following program because the first clause subsumes the second:
\[
\tw{h} = \left\{
\begin{array}{l}
  \tw{p(A):- q(A).}\\
  \tw{p(A):- q(A),r(A).}\\
\end{array}
\right\}
\]
Finally, \name{} uses constraints to improve efficiency (mostly by removing redundancy).
For instance, \name{} uses constraints to use variables in order, which prunes the program \tw{p(B):- q(B)} because we could generate \tw{p(A):- q(A)}.






\begin{figure}[ht]
\centering
\begin{minipage}{.9\linewidth}
\begin{lstlisting}[frame=single]
recursive:- recursive(Clause).

recursive(Clause):- head_literal(Clause,P,A,_), body_literal(Clause,P,A,_).

has_base:- clause(Clause), not recursive(Clause).

% need multiple clauses for recursion
:- recursive(_), not clause(1).

% prevent recursion without a basecase
:- recursive, not has_base.
\end{lstlisting}
\end{minipage}
\caption{
  Constraints used by \name{} to prune invalid recursive programs.
}
\label{fig:recursion}
\end{figure}






\subsubsection{Language bias constraints}
\label{sec:langconstrain}

\name{} supports optional hypothesis constraints to prune the hypothesis space.
Figure \ref{fig:language-constraints} shows example language bias constraints, such as to prevent singleton variables and to enforce Datalog restrictions (where head variables must appear in the body).
Declarative constraints have many benefits, notably the ease to define them.
For instance, to add simple types to \name{} requires the single constraint shown in Figure \ref{fig:language-constraints}.
Through constraints, \name{} also supports the standard notions of \emph{recall} and \emph{input/output}\footnote{
    An input argument specifies that, at the time of calling a predicate, the corresponding argument must be instantiated, which is useful when inducing Prolog programs where literal order matters.
} arguments of mode declarations \cite{progol}.
\name{} also supports \emph{functional} and \emph{irreflexive} constraints, and constraints on recursive programs, such as disallowing left recursion or mutual recursion.
Finally, as we show in Appendix \ref{app:metarules}, \name{} can also use constraints to impose \emph{metarules}, clause templates used by many ILP systems \cite{crop:reduce}, which ensures that each clause in a program is an instance of a metarule.

\begin{figure}[ht]
\centering
\begin{minipage}{1\linewidth}
\begin{lstlisting}[frame=single]
head_var(Clause,Var):- head_literal(Clause,_,_,Vars), var_member(Var,Vars).

body_var(Clause,Var):- body_literal(Clause,_,_,Vars), var_member(Var,Vars).

% prevent singleton variables
:- clause_var(Clause,Var), #count{P,Vars: var_in_literal(Clause,P,Vars,Var)} == 1.

% head vars must appear in the body
:- head_var(Clause,Var), not body_var(Clause,Var).

%% type matching
var_type(Clause,Var,Type):-
    var_in_literal(Clause,P,Vars,Var),
    var_pos(Var,Vars,Pos),
    type(P,Pos,Type).

:- clause_var(Clause,Var), #count{Type : var_type(Clause,Var,Type)} > 1.
\end{lstlisting}
\end{minipage}
\caption{Optional language bias constraints used by \name{}.}
\label{fig:language-constraints}
\end{figure}


\subsubsection{Hypothesis constraints}

As with many ILP systems \cite{progol,aleph,DBLP:conf/ilp/AthakraviABRS14,ilasp,inspire}, \name{} supports \emph{clause} constraints, which allow a user to prune specific clauses from the hypothesis space.
\name{} additionally supports the more general concept of \emph{hypothesis constraints} (Definition \ref{def:hconstraint}), which are defined over a whole program (a set of clauses) rather than a single clause (also employed in previous work \cite{DBLP:conf/ilp/AthakraviABRS14}).
For instance, hypothesis constraints allow us to prune recursive programs that do not contain a base case clause (Figure \ref{fig:recursion}), to prune left recursive or mutually recursive programs, or to prune programs which contain subsumption redundancy between clauses.

As a toy example, suppose you want two disallow two predicate symbols $p/2$ and $q/2$ from both appearing in a program.
Then this hypothesis constraint is trivial to express with \name{}:

\begin{center}
\small
\begin{minipage}{.6\linewidth}
\begin{lstlisting}[frame=single]
:- body_literal(_,p,2,_), body_literal(_,q,2,_).
\end{lstlisting}
\end{minipage}
\end{center}

\noindent
As we show in Appendix \ref{app:metarules}, \name{} can simulate metarules through hypothesis constraints.
We are unaware of any other ILP system that supports hypothesis constraints, at least with the same ease and flexibility as \name{}.
\subsection{Test}

In the test stage, \name{} converts an answer set to a definite program and tests it against the training examples.
As Table \ref{tab:answerset_to_prolog} shows, this conversion is straightforward, except if input/output argument directions are given, in which case \name{} orders the body literals of a clause.
To evaluate a hypothesis, we use a Prolog interpreter.
For each example, \name{} checks whether the example is entailed by the hypothesis and background knowledge.
We enforce a timeout to halt non-terminating programs.
If a hypothesis fails, then \name{} identifies what type of failure has occurred and what constraints to generate (using the failures and constraints from Section \ref{sec:failures}).
\newcommand{\code}[1]{\textbf{\tw{#1}}}

\subsection{Constrain}
\label{sec:constrain}

If a hypothesis fails, then, in the constrain stage, \name{} derives ASP constraints which prune hypotheses, thus constraining subsequent hypothesis generation.
Specifically, we describe how we transform a failed hypothesis (a definite program) to a hypothesis constraint (an ASP constraint written in the encoding from Section \ref{sec:generate}).
We describe the generalisation, specialisation, and elimination constraints that \name{} uses, based on the definitions in Section \ref{sec:failures}.
As our experiments consider a version of \name{} without constraint pruning,
we also describe the \emph{banish} constraint, which prunes one specific hypothesis.
To distinguish between Prolog and ASP code, we represent the code of definite programs in \tw{typewriter} font and ASP code in \code{bold typewriter} font.

\subsubsection{Encoding atoms}

In our encoding, the atom \tw{f(A,B)} is represented as either \code{head\_literal(Clause,f,2,(V0,V1))} or \code{body\_literal(Clause,f,2,(V0,V1))}.
The constant \code{2} is the predicate's arity and the variable \code{Clause} indicates that the clause index is undetermined.
Two functions encode atoms into ASP literals.
The function $\mathit{encodeHead}$ encodes a head atom and $\mathit{encodeBody}$ encodes a body atom.
The first argument specifies the clause the atom belongs to.
The second argument is the atom.
Variables of the atom are converted to variables in our ASP encoding by the $\mathit{encodeVar}$ function.
\[
\begin{array}{l}
\mathit{encodeHead}(Clause,\tw{Pred(Var}_0,\ldots,\tw{Var}_k\tw{)})
:=
\\
\hspace{3ex}
\code{head\_literal(}Clause\code{,}
\tw{Pred}
\code{,}
k+1
\code{,}
\code{(}\mathit{encodeVar}(\tw{Var}_0)\code{,}
\ldots\code{,}
\mathit{encodeVar}(\tw{Var}_k)\code{)}
\code{)}
\\
\\
\mathit{encodeBody}(Clause,\tw{Pred(Var}_0,\ldots,\tw{Var}_k\tw{)})
:=
\\
\hspace{3ex}
\code{body\_literal(}Clause\code{,}
\tw{Pred}
\code{,}
k+1
\code{,}
\code{(}\mathit{encodeVar}(\tw{Var}_0)\code{,}
\ldots\code{,}
\mathit{encodeVar}(\tw{Var}_k)\code{)}
\code{)}
\end{array}
\]

\noindent
For instance, using the term \code{Cl} as a clause variable, calling $\mathit{encodeHead}(\code{Cl},\tw{f(A,B)})$ returns the ASP literal \code{head\_literal(Cl,f,2,(V0,V1))}.
Similarly, calling $\mathit{encodeBody}(\code{Cl},\tw{f(A,B)})$ returns \code{body\_literal(Cl,f,2,(V0,V1))}.

\subsubsection{Encoding clauses}
\label{sec:encoding_clauses}

We encode clauses by building on the encoding of atoms.
Let \code{Cl} be a clause index variable.
Consider the clause \tw{last(A,B):- reverse(A,C),head(C,B)}.
The following ASP literals encode where these atoms occur in a single clause:
\[
\begin{array}{l}
\code{head\_literal(Cl,last,2,(V0,V1))}\\
\code{body\_literal(Cl,reverse,2,(V0,V2))}\\
\code{body\_literal(Cl,head,2,(V2,V1))}
\end{array}
\]

\noindent
An ASP solver will instantiate the variables \code{V0}, \code{V1}, and \code{V2} with indices representing variables of hypotheses, e.g.~\code{0} for \tw{A}, \code{1} for \tw{B}, etc.
Note that the above encoding allows for $\code{V0} = \code{V1} = \code{V2} = \code{0}$, where all the variables are \tw{A}.
To ensure that variables are distinct we need to impose the inequality $\code{V0!=V1}$ and $\code{V0!=V2}$ and $\code{V1!=V2}$.
The function $\mathit{assertDistinct}$ generates such inequalities, one between each pair of variables it is given.
The function $\mathit{encodeClause}$ implements both the straightforward translation and the variable distinctness assertion:
\[
\begin{array}{l}
\mathit{encodeClause}(Clause,(
\tw{head}
\tw{:-}
\tw{body}_1
,\ldots,
\tw{body}_m)
)
:=
\\
\hspace{3ex}
\mathit{encodeHead}(Clause,\tw{head})
\code{,}
\mathit{encodeBody}(Clause,\tw{body}_1)
\code{,}
\ldots
\code{,}
\\\hspace{3ex}
\mathit{encodeBody}(Clause,\tw{body}_m)
\code{,}
\\\hspace{3ex}
\mathit{assertDistinct}(
\mathit{vars}(\tw{head})
\cup
\mathit{vars}(\tw{body}_1)
\cup
\ldots
\cup
\mathit{vars}(\tw{body}_m)
)
\end{array}
\]

\noindent
As clauses can occur in multiple hypotheses, it is convenient to refer to clauses by identifiers.
The function $\mathit{clauseIdent}$ maps clauses to unique ASP constants\footnote{Even though the examples use increasing numbers in the identifiers, $\mathit{clauseIdent}$ can be any injective function, i.e.~always mapping a clause to the same unique identifier.}.
We use the ASP literal $\code{included\_clause(}\mathit{cl}\code{,}\mathit{id}\code{)}$ to represent that a clause with index $\mathit{cl}$ includes all literals of a clause identified by $\mathit{id}$.
The $\mathit{inclusionRule}$ function generates an \emph{inclusion rule}, an ASP rule whose head is true when the literals of the provided clause occur together in a clause:
\[
\begin{array}{l}
\mathit{inclusionRule}(
\tw{head}
\tw{:-}
\tw{body}_1
,\ldots,
\tw{body}_m
)
:=
\\
\hspace{3ex}
\code{included\_clause(}
\code{Cl}
\code{,}
clauseIdent(
\tw{head}
\tw{:-}
\tw{body}_1
,\ldots,
\tw{body}_m)
\code{):-}\\
\hspace{6ex}
\mathit{encodeClause}(\code{Cl},
(\tw{head}
\tw{:-}
\tw{body}_1
,\ldots,
\tw{body}_m)
)
\code{.}
\end{array}
\]

\noindent
Suppose that $\mathit{clauseIdent}(\tw{last(A,B):- reverse(A,C),head(C,B)}) = \code{id\textsubscript{1}}$.
Then the rule obtained by $\mathit{inclusionRule}(\tw{last(A,B):- reverse(A,C),head(C,B))}$ is:
\[
\begin{array}{l}
\code{included\_clause(Cl,id\textsubscript{1})}
\code{:-}
\\
\hspace{3ex}
\code{head\_literal(Cl,last,2,(V0,V1)),}
\\
\hspace{3ex}
\code{body\_literal(Cl,reverse,2,(V0,V2)),}
\\
\hspace{3ex}
\code{body\_literal(Cl,head,2,(V2,V1)),}
\\
\hspace{3ex}
\code{V0!=V1}
\code{,}
\code{V0!=V2}
\code{,}
\code{V1!=V2}
\code{.}
\end{array}
\]

\noindent
Note that $\code{included\_clause(}\mathit{cl}\code{,}\mathit{id}\code{)}$ being true does not mean that other literals do \emph{not} occur in the clause.
For example, if a clause with index \code{0} encoded the clause \tw{last(A,B):- reverse(A,C),head(C,B),tail(C,A)}, then $\code{included\_clause(0,id\textsubscript{1})}$ would also hold.

In our encoding, $\code{clause\_size(}\mathit{cl}\code{,}\mathit{m}\code{)}$ is only true when clause $\mathit{cl}$ has exactly $\mathit{m}$ body literals.
Hence when literals \code{included\_clause(0,id\textsubscript{1})} and \code{clause\_size(0,2)} are both true, the clause with index \code{0} exactly encodes
\tw{last(A,B):- reverse(A,C),head(C,B)}.
The function $\mathit{exactClause}$ derives a pair of ASP literals checking that a clause occurs exactly:
\[
\begin{array}{l}
\mathit{exactClause}(
Clause,
(
\tw{head}
\tw{:-}
\tw{body}_1
,\ldots,
\tw{body}_m
)
)
:=
\\
\hspace{3ex}
\code{included\_clause(}
Clause
\code{,}
clauseIdent(
\tw{head}
\tw{:-}
\tw{body}_1
,\ldots,
\tw{body}_m)
\code{)}
\code{,}
\\
\hspace{3ex}
\code{clause\_size(}
Clause
\code{,}
m
\code{)}
\end{array}
\]

%
%

\subsubsection{Generalisation constraints}

Given a hypothesis $H$, by Definition \ref{def:generalisation}, any hypothesis that includes all of $H$'s clauses exactly is a generalisation of $H$.
We use this fact to define function $\mathit{generalisationConstraint}$, which converts a set of clauses into ASP encoded clause inclusion checking rules as well as a generalisation constraint (Definition \ref{def:generalisation_constraint}).
We use $\mathit{exactClause}$ to impose that a clause is not specialised.
Each clause is given its own ASP variable, meaning that the clauses can occur in any order.
\[
\begin{array}{l}
\mathit{generalisationConstraint}(\{\tw{Clause}_0,\ldots,\tw{Clause}_{n-1}\})
:=
\\
\hspace{3ex}
inclusionRule(
\tw{Clause}_0
)
\\
\hspace{3ex}
\ldots
\\
\hspace{3ex}
inclusionRule(
\tw{Clause}_{n-1}
)
\\
\hspace{3ex}
\code{:-}~
\mathit{exactClause}(\code{Cl\textsubscript{0}},\tw{Clause}_0)
\code{,}
\ldots
\code{,}
\mathit{exactClause}(\code{Cl\textsubscript{n-1}},\tw{Clause}_{n-1})
\code{.}
\end{array}
\]

\noindent
Figure \ref{fig:generalisation-constraint} illustrates $\mathit{generalisationConstraint}$ deriving both an inclusion rule and a generalisation constraint.
\begin{figure}[ht]
\centering
\begin{minipage}{0.5\textwidth}%
\[
\tw{h} = \left\{
\begin{array}{l}
\tw{last(A,B):- reverse(A,C),head(C,B).}
\end{array}
\right\}
\]
\end{minipage}%
\begin{minipage}{0.5\textwidth}%
\begin{lstlisting}[frame=single]
included_clause(Cl,id1):-
  head_literal(Cl,last,2,(V0,V1)),
  body_literal(Cl,reverse,2,(V0,V2)),
  body_literal(Cl,head,2,(V2,V1)),
  V0!=V1,V0!=V2,V1!=V2.
:-
  included_clause(Cl0,id1),
  clause_size(Cl0,2).
\end{lstlisting}
\end{minipage}%
\caption{
ASP encoded inclusion rule and generalisation constraint for the hypothesis \tw{h}.
}
\label{fig:generalisation-constraint}
\end{figure}

\subsubsection{Specialisation constraints}
Given a hypothesis $H$, by Definition \ref{def:specialisation}, any hypothesis which has every clause of $H$ occur, where each of these clauses may be specialised, and includes no other clauses, is a specialisation of $H$.
The function $\mathit{specialisationConstraint}$ uses this fact to derive an ASP encoded specialisation constraint (Definition \ref{def:specialisation_constraint}) alongside inclusion rules.
When $\code{included\_clause(}\mathit{cl}\code{,}\mathit{id}\code{)}$ is true, additional atoms can occur in the clause $\mathit{cl}$.
The literal $\code{not clause(}n\code{)}$ ensures that no additional clause is added to the $n$ distinct clauses of the provided hypothesis.
\[
\begin{array}{l}
\mathit{specialisationConstraint}(\{\tw{Clause}_0,\ldots,\tw{Clause}_{n-1}\})
:=
\\
\hspace{3ex}
inclusionRule(
\tw{Clause}_0
)
\\
\hspace{3ex}
\ldots
\\
\hspace{3ex}
inclusionRule(
\tw{Clause}_{n-1}
)
\\
\hspace{3ex}
\code{:-}~
\code{included\_clause(\code{Cl\textsubscript{0}},}
\mathit{clauseIdent}(
\tw{Clause}_0
)
\code{)}
\code{,}
\ldots
\code{,}
\\
\hspace{5.8ex}
\code{included\_clause(\code{Cl\textsubscript{n-1}},}
\mathit{clauseIdent}(
\tw{Clause}_{n-1}
)
\code{)}
\code{,}
\\
\hspace{5.8ex}
\mathit{assertDistinct}(\{ \code{Cl\textsubscript{0}},\ldots,\code{Cl\textsubscript{n-1}}\})
\code{,}
\code{not clause(}
n
\code{)}
\code{.}
\\
\end{array}
\]
\noindent
We illustrate why asserting that specialised clauses are distinct is necessary.
Consider the hypotheses \tw{h$_1$} and \tw{h$_2$}:
\[
\tw{h$_1$} = \left\{
\begin{array}{l}
  \tw{last(A,B):- head(A,B).}\\
  \tw{last(A,B):- sumlist(A,B).}\\
\end{array}
\right\}\\
\tw{h$_2$} = \left\{
\begin{array}{l}
  \tw{last(A,B):- head(A,B),sumlist(A,B).}\\
  \tw{last(A,B):- member(A,B).}\\
\end{array}
\right\}\\
\]
The first clause of \tw{h$_2$} specialises both clauses in \tw{h$_1$}, yet \tw{h$_2$} is not a specialisation of \tw{h$_1$}.
According to Definition \ref{def:specialisation}, \emph{each} clause needs to be subsumed by a provided clause.
Note that $\mathit{specialisationConstraint}$ only considers hypotheses with at most $n$ clauses.
It is not possible for one of these clauses to be non-specialising, as each of the original $n$ clauses is required to be specialised by a distinct clause.

Figure \ref{fig:specialisation-constraint} illustrates a specialisation constraint derived by $\mathit{specialisationConstraint}$.

\begin{figure}[ht]
\centering
\begin{minipage}{0.5\textwidth}%
\[
\tw{h} = \left\{
\begin{array}{l}
  \tw{rev(A,B):- head(A,B).}\\
  \tw{rev(A,B):- tail(A,C),head(C,B).}
\end{array}
\right\}
\]
\end{minipage}%
\begin{minipage}{0.5\textwidth}%
\begin{lstlisting}[frame=single]
included_clause(Cl,id2):-
  head_literal(Cl,rev,2,(V0,V1)),
  body_literal(Cl,head,2,(V0,V1)),
  V0!=V1.
included_clause(Cl,id3):-
  head_literal(Cl,rev,2,(V0,V1)),
  body_literal(Cl,tail,2,(V0,V2)),
  body_literal(Cl,head,2,(V2,V1)),
  V0!=V1,V0!=V2,V1!=V2.
:-
  included_clause(Cl0,id2),
  included_clause(Cl1,id3),
  Cl0!=Cl1,not clause(2).
\end{lstlisting}
\end{minipage}%
\caption{
ASP encoded inclusion rules and specialisation constraint for the hypothesis \tw{h}.
}
\label{fig:specialisation-constraint}
\end{figure}

\subsubsection{Elimination constraints}

By Proposition \ref{prop:elim_sound}, given a totally incomplete hypothesis $H$, any separable hypothesis which includes all of $H$'s clauses, where each clause may be specialised, cannot be an optimal solution.
We add the following code to the \name{} encoding to detect separable hypotheses:

\begin{center}
\begin{minipage}{.4\linewidth}
\small
\begin{lstlisting}[frame=single]
non_separable:-
    head_literal(_,P,A,_),
    body_literal(_,P,A,_).

separable:-
    not non_separable.
\end{lstlisting}
\end{minipage}
\end{center}

\noindent
The function $\mathit{eliminationConstraint}$ uses this fact to derive an ASP encoded elimination constraint (Definition \ref{def:elimination_constraint}).
As in $\mathit{specialisationConstraint}$, $\code{included\_clause(}\mathit{cl}\code{,}\mathit{id}\code{)}$ is used to allow additional literals in clauses, ensuring that provided clauses are specialised.
However, $\mathit{eliminationConstraint}$ does not require that every clause is a specialisation of a provided clause.
Instead, all that is required is that the hypothesis is separable.
\[
\begin{array}{l}
\mathit{eliminationConstraint}(\{\tw{Clause}_0,\ldots,\tw{Clause}_{n-1}\})
:=
\\
\hspace{3ex}
inclusionRule(
\tw{Clause}_0
)
\\
\hspace{3ex}
\ldots
\\
\hspace{3ex}
inclusionRule(
\tw{Clause}_{n-1}
)
\\
\hspace{3ex}
\code{:-}~
\code{included\_clause(\code{Cl\textsubscript{0}},}
\mathit{clauseIdent}(
\tw{Clause}_0
)
\code{)}
\code{,}
\ldots
\code{,}
\\
\hspace{5.8ex}
\code{included\_clause(\code{Cl\textsubscript{n-1}},}
\mathit{clauseIdent}(
\tw{Clause}_{n-1}
)
\code{)}
\code{,}
\\
\hspace{5.8ex}
\code{separable}
\code{.}
\\
\end{array}
\]
Figure \ref{fig:elimination_constraint} illustrates an elimination constraint derived by $\mathit{eliminationConstraint}$.

\begin{figure}[ht]
\centering
\begin{minipage}{0.5\textwidth}%
\[
\tw{h} = \left\{
\begin{array}{l}
\tw{last(A,B):- tail(A,C),head(C,B).}
\end{array}
\right\}
\]
\end{minipage}%
\begin{minipage}{0.5\textwidth}%
\begin{lstlisting}[frame=single]
included_clause(Cl,id4):-
  head_literal(Cl,last,2,(V0,V1)),
  body_literal(Cl,tail,2,(V0,V2)),
  body_literal(Cl,head,2,(V2,V1)),
  V0!=V1,V0!=V2,V1!=V2.
:-
  included_clause(Cl0,id4),
  separable.
\end{lstlisting}
\end{minipage}%
\caption{
ASP encoded inclusion rule and elimination constraint for the hypothesis \tw{h}.
}
\label{fig:elimination_constraint}
\end{figure}

\subsubsection{Banish constraints}

In the experiments we compare \name{} against a version of itself without constraint pruning.
To do so we need to remove single hypotheses from the hypothesis space.
We introduce the \emph{banish constraint} for this purpose.
To prune a specific hypothesis, hypotheses with different variables should not be pruned.
We accomplish this condition by changing the behaviour of the $\mathit{encodeVar}$ function.
Normally $\mathit{encodeVar}$ returns ASP variables which are then grounded to indices that correspond to the variables of hypotheses.
Instead, by the following definition, $\mathit{encodeVar}$ directly assigns the corresponding index for a hypothesis variable:
\[
\mathit{encodeVar} =
\{
\,
\tw{A} \mapsto \code{0};\,
\tw{B} \mapsto \code{1};\,
\tw{C} \mapsto \code{2};\,
\ldots
\,
\}
\]

\noindent
For a banish constraint no additional literals in clauses are allowed, nor are additional clauses.
The below function $\mathit{banishConstraint}$ ensures both conditions when converting a hypothesis to an ASP encoded banish constraint.
That provided clauses occur non-specialised is ensured by $\mathit{exactClause}$.
The literal $\code{not clause(}n\code{)}$ asserts that there are no more clauses than the original number.
\[
\begin{array}{l}
\mathit{banishConstraint}(\{\tw{Clause}_0,\ldots,\tw{Clause}_{n-1}\})
:=
\\
\hspace{3ex}
inclusionRule(
\tw{Clause}_0
)
\\
\hspace{3ex}
\ldots
\\
\hspace{3ex}
inclusionRule(
\tw{Clause}_{n-1}
)
\\
\hspace{3ex}
\code{:-}~
\mathit{exactClause}(\code{Cl\textsubscript{0}},\tw{Clause}_0)
\code{,}
\ldots
\code{,}
\mathit{exactClause}(\code{Cl\textsubscript{n-1}},\tw{Clause}_{n-1})
\code{,}
\\
\hspace{5.8ex}
\code{not clause(}
n
\code{)}
\code{.}
\end{array}
\]
\noindent
Figure \ref{fig:banish_constraint} illustrates a banish constraint derived by $\mathit{banishConstraint}$.
\begin{figure}[ht]
\centering
\begin{minipage}{0.45\textwidth}%
\[
\tw{h} = \left\{
\begin{array}{l}
\tw{f(A):- head(A,B),one(B).}\\
\tw{f(A):- tail(A,B),empty(B).}
\end{array}
\right\}
\]
\end{minipage}%
\hspace{0.04\textwidth}
\begin{minipage}{0.46\textwidth}%
\begin{lstlisting}[frame=single]
included_clause(Cl,id5):-
  head_literal(Cl,f,1,(0,)),
  body_literal(Cl,head,2,(0,1)),
  body_literal(Cl,one,1,(1,)).
included_clause(Cl,id6):-
  head_literal(Cl,f,1,(0,)),
  body_literal(Cl,tail,2,(0,1)),
  body_literal(Cl,empty,1,(1,)).
:-
  included_clause(Cl0,id5),
  clause_size(Cl0,2),
  included_clause(Cl1,id6),
  clause_size(Cl1,2),
  not clause(2).
\end{lstlisting}
\end{minipage}%
\caption{
ASP encoded inclusion rules and banish constraint for the hypothesis \tw{h}.
}
\label{fig:banish_constraint}
\end{figure}
\subsection{\name{} loop and multi-shot solving}
\label{sec:loop}

A naive implementation of Algorithm \ref{alg:popper}, such as performing iterative deepening on the program size, would duplicate grounding and solving during the generate step.
To improve efficiency, we use Clingo's multi-shot solving \cite{multishot} to maintain state between the three stages.
The idea of multi-shot solving is that state of the solving process for an ASP program can be saved to help solve modifications of that program.
The essence of the multi-shot cycle is that a ground program is given to an ASP solver, yielding an answer set, who's processing leads to a (first-order) extension of the program.
Only this extension then needs grounding and adding to the running ASP instance, which means that the running solver may, for example, maintain learned conflicts.

\name{} uses multi-shot solving as follows.
The initial ASP program is the encoding described in Section \ref{sec:generate}.
\name{} asks Clingo to ground the initial program and prepare for its solving.
In the generate stage, the solver is asked to return an answer set, i.e.~a model, of the current program.
\name{} converts such an answer set to a definite program and tests it against the examples.
If a hypothesis fails, \name{} generates ASP constraints using the functions in Section \ref{sec:constrain} and adds them to the running Clingo instance, which grounds the constraints and adds the new (propositional) rules to the running solver.
We employ a hard constraint on the program size that reasons about an \emph{external atom} \cite{multishot} \emph{size(N)}.
Initially, programs need to consist of just one literal.
When there are no more answer sets, we increment the program size.
Every time we increment the program size, e.g.~from $N$ to $N{+}1$,
we add a new atom \emph{size(N+1)} and a new constraint enforcing this program size.
Only the new constraint is ground at this point.
We disable the previous constraint by setting the external atom \emph{size(N)} to false.
The solver knows which parts of the search space (i.e.~hypothesis space) have already been considered and will not revisit them.
This loop repeats until either (i) \name{} finds an optimal solution, or (ii) there are no more hypotheses to test.

\subsection{Worked example}
To illustrate \name{}, reconsider the example from the introduction of learning a \emph{last/2} hypothesis to find the last element of a list.
For simplicity, assume an initial hypothesis space $\mathcal{H}_1$:

\[
\mathcal{H}_1 = \left\{
\begin{array}{l}
\tw{h$_1$} = \left\{
\begin{array}{l}
    \tw{last(A,B):- head(A,B).}\\
\end{array}
\right\}\\
\tw{h$_2$} = \left\{
\begin{array}{l}
    \tw{last(A,B):- head(A,B),empty(A).}\\
\end{array}
\right\}\\
\tw{h$_3$} = \left\{
\begin{array}{l}
    \tw{last(A,B):- tail(A,C),head(C,B).}\\
\end{array}
\right\}\\
\tw{h$_4$} = \left\{
\begin{array}{l}
    \tw{last(A,B):- reverse(A,C),head(C,B).}\\
\end{array}
\right\}\\
\tw{h$_5$} = \left\{
\begin{array}{l}
    \tw{last(A,B):- head(A,B),reverse(A,C),head(C,B).}\\
\end{array}
\right\}\\
\tw{h$_6$} = \left\{
\begin{array}{l}
    \tw{last(A,B):- tail(A,C),head(C,B).}\\
    \tw{last(A,B):- reverse(A,C),head(C,B).}\\
\end{array}
\right\}\\
\tw{h$_7$} = \left\{
\begin{array}{l}
    \tw{last(A,B):- tail(A,C),head(C,B).}\\
    \tw{last(A,B):- tail(A,C),tail(C,D),head(D,B).}\\
\end{array}
\right\}\\
\tw{h$_8$} = \left\{
\begin{array}{l}
    \tw{last(A,B):- reverse(A,C),tail(C,D),head(D,B).}\\
    \tw{last(A,B):- tail(A,C),reverse(C,D),head(D,B).}\\
\end{array}
\right\}\\
\tw{h$_9$} = \left\{
\begin{array}{l}
    \tw{last(A,B):- head(A,B).}\\
    \tw{last(A,B):- reverse(A,C),tail(C,D),head(D,B).}\\
    \tw{last(A,B):- tail(A,C),reverse(C,D),head(D,B).}\\
\end{array}
\right\}\\
\end{array}
\right\}
\]

\noindent
Also assume we have the positive ($E^+$) and negative ($E^-$) examples:
\[
E^+ = \left\{
\begin{array}{l}
\tw{last([l,a,u,r,a],a).}\\
\tw{last([p,e,n,e,l,o,p,e],e).}
\end{array}
\right\}
\hspace{3ex}
E^- = \left\{
\begin{array}{l}
\tw{last([e,m,m,a],m).}\\
\tw{last([j,a,m,e,s],e).}
\end{array}
\right\}
\]

\noindent
To start, \name{} generates the simplest hypothesis:
\[
\tw{h$_1$} = \left\{
\begin{array}{l}
    \tw{last(A,B):- head(A,B).}\\
\end{array}
\right\}
\]
\noindent
\name{} then tests \tw{h$_1$} against the examples and finds that it \emph{fails} because it does not entail any positive example and is therefore too \emph{specific}.
\name{} then generates a specialisation constraint to prune specialisations of \tw{h$_1$}:

\begin{center}
\small
\begin{minipage}{0.42\textwidth}%
\begin{lstlisting}[frame=single]
included_clause(Cl,id1):-
    head_literal(Cl,last,2,(V0,V1)),
    body_literal(Cl,head,2,(V0,V1)),
    V0!=V1.
:-
    included_clause(Cl0,id1),
    not clause(1).
\end{lstlisting}
\end{minipage}
\end{center}

\noindent
\name{} adds this constraint to the meta-level ASP program which prunes \tw{h$_2$} and \tw{h$_5$} from the hypothesis space.
In addition, because \tw{h$_1$} does not entail any positive example (is \emph{totally} incomplete), \name{} also generates an elimination constraint:

\begin{center}
\small
\begin{minipage}{0.42\textwidth}%
\begin{lstlisting}[frame=single]
:-
    included_clause(Cl0,id1),
    separable.
\end{lstlisting}
\end{minipage}
\end{center}

\noindent
\name{} adds this constraint to the meta-level ASP program which prunes \tw{h$_9$} from the hypothesis space.
The hypothesis space is now:
\[
\mathcal{H}_2 = \left\{
\begin{array}{l}
\tw{h$_3$} = \left\{
\begin{array}{l}
    \tw{last(A,B):- tail(A,C),head(C,B).}\\
\end{array}
\right\}\\
\tw{h$_4$} = \left\{
\begin{array}{l}
    \tw{last(A,B):- reverse(A,C),head(C,B).}\\
\end{array}
\right\}\\
\tw{h$_6$} = \left\{
\begin{array}{l}
    \tw{last(A,B):- tail(A,C),head(C,B).}\\
    \tw{last(A,B):- reverse(A,C),head(C,B).}\\
\end{array}
\right\}\\
\tw{h$_7$} = \left\{
\begin{array}{l}
    \tw{last(A,B):- tail(A,C),head(C,B).}\\
    \tw{last(A,B):- tail(A,C),tail(C,D),head(D,B).}\\
\end{array}
\right\}\\
\tw{h$_8$} = \left\{
\begin{array}{l}
    \tw{last(A,B):- reverse(A,C),tail(C,D),head(D,B).}\\
    \tw{last(A,B):- tail(A,C),reverse(C,D),head(D,B).}\\
\end{array}
\right\}
\end{array}
\right\}
\]
\noindent
\name{} generates another hypothesis (\tw{h$_3$}) and tests against the examples and finds that it fails because it entails the negative example \tw{last([e,m,m,a],m)} and is therefore too \emph{general}.
\name{} then generates a generalisation constraint to prune generalisations of \tw{h$_3$}:

\begin{center}
\small
\begin{minipage}{0.42\textwidth}%
\begin{lstlisting}[frame=single]
included_clause(Cl,id2):-
    head_literal(Cl,last,2,(V0,V1)),
    body_literal(Cl,tail,2,(V0,V2)),
    body_literal(Cl,head,2,(V2,V1)),
    V0!=V1,V0!=V2,V1!=V2.
:-
  included_clause(Cl0,id2),
  clause_size(Cl0,2).
\end{lstlisting}
\end{minipage}%
\end{center}

\noindent
\name{} adds this constraint to the meta-level ASP program which prunes \tw{h$_6$} and \tw{h$_7$} from the hypothesis space.
The hypothesis space is now:
\[
\mathcal{H}_3 = \left\{
\begin{array}{l}
\tw{h$_4$} = \left\{
\begin{array}{l}
    \tw{last(A,B):- reverse(A,C),head(C,B).}\\
\end{array}
\right\}\\
\tw{h$_8$} = \left\{
\begin{array}{l}
    \tw{last(A,B):- reverse(A,C),tail(C,D),head(D,B).}\\
    \tw{last(A,B):- tail(A,C),reverse(C,D),head(D,B).}\\
\end{array}
\right\}
\end{array}
\right\}
\]
\noindent
Finally, \name{} generates another hypothesis (\tw{h$_4$}), tests it against the examples, finds that it does not fail, and returns it.
\subsection{Correctness}
\label{sec:correctness}

We now show the correctness of \name{}.
However, we can only show this result for when the hypothesis space only contains decidable programs, e.g.~Datalog programs.
When the hypothesis space contains arbitrary definite programs, then the results do not hold because checking for entailment of an arbitrary definite program is only semi-decidable \cite{tarnlund:hornclause}.
In other words, the results in this section only hold when every hypothesis in the hypothesis space is guaranteed to terminate\footnote{
In practice, such as in our experiments on learning list transformation programs, we enforce a timeout when testing programs, i.e. we assume that every solution terminates before the timeout.
}.

We first show that \name{}'s base encoding (Figure \ref{fig:alan}) can generate every declaration consistent hypothesis (Definition \ref{def:decl_cons_hypothesis}):

\begin{proposition}
\label{prop:all_hypotheses}
The base encoding of \name{} has a model for every declaration consistent hypothesis.
\end{proposition}

\noindent
\begin{proof}
Let $D = (D_h, D_b)$ be a declaration bias, $N_{var}$ be the maximum number of unique variables, $N_{body}$ be the maximum number of body literals, $N_{clause}$ be the maximum number of clauses, $H$ be any hypothesis declaration consistent with $D$ and these parameters, and $C$ be any clause in $H$.
Our encoding represents the head literal $p_h(H_1,\dots,H_n)$ of $C$ as a choice literal \tw{head\_literal(i,$p_h$,$n$,($H_1$,$\dots$,$H_n$))} guarded by the condition \tw{head\_pred($p_h$,$n$)} $\in D_h$, which clearly holds.
Our encoding represents a body literal $p_b(B_1,\dots,B_m)$ of $C$ as a choice literal \tw{body\_literal($i$,$p_b$,$m$,($B_1$,$\ldots$,$B_m$))} guarded by the condition \tw{body\_pred($p_b$,$m$)} $\in D_b$, which clearly holds.
The base encoding only constrains the above guesses by three conditions:
(i) at most $N_{var}$ unique variables per clause,
(ii) at least 1 and at most $N_{body}$ body literals per clause, and
(iii) at most $N_{clause}$ clauses.
As both the hypothesis and the guessed literals satisfy the same conditions, we conclude there exists a model representing $H$.
\end{proof}

\noindent
We show that any hypothesis returned by \name{} is a solution (Definition \ref{def:solution}):

\begin{proposition}[Soundness]
\label{prop:poppersound}
Any hypothesis returned by \name{} is a solution.
\end{proposition}
\begin{proof}
Any returned hypothesis has been tested against the training examples and confirmed as a solution.
\end{proof}

\noindent
To make the next two results shorter, we introduce a lemma to show that \name{} never prunes optimal solutions (Definition \ref{def:opthyp}):

\begin{lemma}
\label{lemma:prune}
\name{} never prunes optimal solutions.
\end{lemma}
\begin{proof}
\name{} only learns constraints from a failed hypothesis, i.e.~a hypothesis that is incomplete or inconsistent.
Let $H$ be a failed hypothesis.
If $H$ is incomplete, then, as described in Section \ref{sec:constrain}, \name{} prunes specialisations of $H$.
Proposition \ref{prop:specialisation_soundness} shows that a specialisation constraint never prunes complete hypotheses, and thus never prunes optimal solutions.
If $H$ is inconsistent, then, as described in Section \ref{sec:constrain}, \name{} prunes generalisations of $H$.
Proposition \ref{prop:generalisation_soundness} shows that a generalisation constraint never prunes consistent hypotheses, and thus never prunes optimal solutions.
Finally, if $H$ is totally incomplete, then, as described in Section \ref{sec:constrain}, \name{} uses an elimination constraint to prune all separable hypotheses that contain $H$.
Proposition \ref{prop:elim_sound} shows that an elimination constraint never prunes optimal solutions.
Since \name{} only uses these three constraints, it never prunes optimal solutions.
\end{proof}

\noindent
We now show that \name{} returns a solution if one exists:

\begin{proposition}[Completeness]
\label{prop:completeness}
\name{} returns a solution if one exists.
\end{proposition}
\begin{proof}
Assume, for contradiction, that \name{} does not return a solution, which implies that (1) \name{} returned a hypothesis that is not a solution, or (2) \name{} did not return a solution.
Case (1) cannot hold because Proposition \ref{prop:poppersound} shows that every hypothesis returned by \name{} is a solution.
For case (2), by Proposition \ref{prop:all_hypotheses}, \name{} can generate every hypothesis so it must be the case that (i) \name{} did not terminate, (ii) a solution did not pass the test stage, or (iii) that every solution was incorrectly pruned.
Case (i) cannot hold because Proposition \ref{prop:hspace} shows that the hypothesis space is finite so there are finitely many hypotheses to generate and test.
Case (ii) cannot hold because a solution is by definition a hypothesis that passes the test stage.
Case (iii) cannot hold because Lemma \ref{lemma:prune} shows that \name{} never prunes optimal solutions.
These cases are exhaustive, so the assumption cannot hold, and thus \name{} returns a solution if one exists.
\end{proof}

\noindent
We show that \name{} returns an optimal solution if one exists:

\begin{theorem}[Optimality]
\label{thm:optimal}
\name{} returns an optimal solution if one exists.
\end{theorem}

\begin{proof}
By Proposition \ref{prop:completeness}, \name{} returns a solution if one exists.
Let $H$ be the solution returned by \name{}.
Assume, for contradiction, that $H$ is not an optimal solution.
By Definition \ref{def:opthyp}, this assumption implies that either (1) $H$ is not a solution, or (2) $H$ is a non-optimal solution.
Case (1) cannot hold because $H$ is a solution.
Therefore, case (2) must hold, i.e.~there must be at least one smaller solution than $H$.
Let $H'$ be an optimal solution, for which we know $size(H') < size(H)$.
By Proposition \ref{prop:all_hypotheses}, \name{} generates every hypothesis, and \name{} generates hypotheses of increasing size (Algorithm \ref{alg:popper}), therefore the smaller solution $H'$ must have been considered before $H$, which implies that $H'$ must have been pruned by a constraint.
However, Lemma \ref{lemma:prune} shows that $H'$ could not have been pruned and so cannot exist, which contradicts the assumption and completes the proof.
\end{proof}


\section{Experiments}
\label{sec:experiments}

We now evaluate \name{}.
\name{} learns constraints from failed hypotheses to prune the hypothesis space to improve learning performance.
We therefore claim that, compared to unconstrained learning, constraints can improve learning performance.
One may think that this improvement is obvious, i.e. constraints will definitely improve performance.
However, it is unclear whether in practice, and if so by how much, constraints will improve learning performance because \name{} needs to (i) analyse failed hypotheses, (ii) generate constraints from them, and (iii) pass the constraints to the ASP system, which then needs to ground and solve them, which may all have non-trivial computational overheads.
Our experiments therefore aim to answer the question:

\begin{description}
    \item [\textbf{Q1}] Can constraints improve learning performance compared to unconstrained learning?
\end{description}

\noindent
To answer this question, we compare \name{} with and without the constrain stage.
In other words, we compare \name{} against a brute-force generate and test approach.
To do so, we use a version of \name{} with only banish constraints enabled to prevent repeated generation of a failed hypothesis.
We call this system \emph{\alan{}}.


Proposition \ref{prop:hspace} shows that the size of the learning from failures hypothesis space is a function of many parameters, including the number of predicate declarations, the number of unique variables in a clause, and the number of clauses in a hypothesis.
To explore this result, our experiments aim to answer the question:

\begin{description}
    \item [\textbf{Q2}] How well does \name{} scale?
\end{description}

\noindent
To answer this question, we evaluate \name{} when varying (i) the size of the optimal solution, (ii) the number of predicate declarations, (iii) the number of constants in the problem, (iv) the number of unique variables in a clause, (v) the maximum number of literals in a clause, and (vi) the maximum number of clauses allowed in a hypothesis.

We also compare \name{} against existing ILP systems.
Our experiments therefore aim to answer the question:

\begin{description}
    \item [\textbf{Q3}] How well does \name{} perform compared to other ILP systems?
\end{description}

\noindent
To answer this question, we compare \name{} against Aleph \cite{aleph}, Metagol, ILASP2i \cite{law:context}, and ILASP3 \cite{law:thesis}.
It is, however, important to note that a direct comparison of ILP systems is difficult because different systems excel at different problems and often employ different biases.
For instance, directly comparing the Prolog-based Metagol against the ASP-based ILASP is difficult because Metagol is often used to learn recursive list manipulation programs, such as string transformations and sorting algorithms, whereas ILASP does not support explicit lists because the ASP system Clingo \cite{clingo}, on which ILASP is built, disallows explicit lists.
Likewise, Aleph and ILASP3 support noise, whereas Metagol and Popper do not.
Moreover, because ILP systems have many learning parameters, it is often possible to show that there exist some parameter settings for which system X can perform better than system Y on a particular problem.
Overall, a direct comparison between ILP systems is difficult, so a reader should not interpret the results as system X is better than system Y.

\subsection{Buttons}
\label{sec:buttons}

The purpose of this first experiment is to evaluate how well \name{} scales when varying the optimal solution size\footnote{Note that, in this experiment, increasing the optimal solution size almost always also increases the size of the hypothesis space for the considered ILP systems.}.
We therefore need a problem where we can control the optimal solution size.
We consider a problem loosely based on the IGGP game \emph{buttons and lights} \cite{iggp}.
The problem is purposely simple: given $p$ buttons, learn which $n$ buttons need to be pressed to win.
For instance, for $n=3$, a solution could be:

\cpl{win(A):- button6(A),button4(A),button7(A)}

\noindent
The variable \emph{A} denotes the player and \emph{button$_p$} denotes that player \emph{A} pressed \emph{button$_p$}.

In this experiment, we fix $p$, the number of buttons, and vary $n$, the number of buttons that need to be pressed, which directly corresponds to the optimal solution size.

\subsubsection{Materials}
We consider two variations of the problem where $p=20$ and $p=200$, which we name \emph{small} and \emph{big} respectively.
We compare \name{}, \alan{}, Metagol, ILASP2i, ILASP3, and Aleph.
To compare the systems, we try to use settings so that each system searches approximately the same hypothesis space.
However, ensuring that the systems search identical hypothesis spaces is near impossible.
For instance, Metagol performs automatic predicate invention and so considers a different hypothesis space to the other systems.
The exact language biases used are in Appendix \ref{app:buttons}.

\paragraph{ILASP settings.}
We asked Mark Law, the ILASP author, for advice on how best to solve this problem with ILASP2i and ILASP3\footnote{Mark suggested an alternative representation that corresponds to learning the negation of the concept, which would have been much more suitable for ILASP. However, this alternative different representation requires NAF which not all of the other systems support.
}.
We run both ILASP2i and ILASP3 with the same settings so we simply refer to both as ILASP.
We run ILASP with the `no constraints', `no aggregates', `disable implication', `disable propagation', and `simple contexts' flags.
We tell ILASP that each BK relation is positive, which prevents it from generating body literals using negation.
We also make the problem propositional and use context-dependent examples \cite{law:context} where the context-dependent BK for each example contains the buttons pressed in that example.
We initially tried to run ILASP with at most ten body literals (`-ml=10' and `--max-rule-length=11') but when given this parameter ILASP would not terminate in the time limit because it pre-computes every rule in the hypothesis space.
Therefore, for each number of buttons $n$, we set the maximum number of body literals to $n$ (`-ml=n' and `--max-rule-length=n+1'), to ensure that ILASP terminates on some of the problems.

\paragraph{Metagol settings.}
Metagol needs metarules (Section \ref{sec:langbias}) to guide the proof search.
We provide Metagol with the following two metarules:
\[
    \begin{array}{l}
    \tw{P(A):- Q(A).}\\
    \tw{P(A):- Q(A),R(A).}\\
    \end{array}
  \]
\noindent

\paragraph{\name{} and \alan{} settings.}
We set \name{} and \alan{} to use at most 1 unique variable, at most 1 clause, and at most $n$ body literals.
These settings match those imposed by Metagol's metarules and somewhat ILASP's propositional representation.
We restrict the clause to have at most $n$ body literals to match ILASP's settings.
When allowed up to ten body literals, Popper performs almost identically.

\paragraph{Aleph settings.}
We also set the maximum number of nodes to be search to be 5000.
As with \name{}, \alan{}, and ILASP, we increase the maximum clause length for Aleph for each value $n$.

\subsubsection{Methods}

For each $n$ in $\{1,2,\dots,10\}$, we generate $200$ positive and $200$ negative examples.
A positive example is a player that has pressed the correct $n$ buttons.
To generate a positive example we sample without replacement $n$ integers from the set $\{1,\dots,p\}$ which correspond to the $n$ buttons that \emph{must} be pressed.
We additionally sample extra buttons that are also pressed, but which are not necessarily pressed in all the positive examples.
A negative example is a player that has not pressed the correct $n$ buttons.
To generate a negative example we sample without replacement at most $n-1$ buttons from the set that must be pressed.
We then sample other buttons that should not be pressed.
By including all $n$ negative examples with $n-1$ correct buttons we guarantee that there is only one correct solution.
We measure learning time as the time to learn a solution.
We enforce a timeout of one minute per task.
We repeat each experiment ten times and plot the standard error.

\subsubsection{Results}
Figure \ref{fig:buttonsa} shows that \name{} clearly outperforms \alan{} on both datasets.
On the small dataset ($p=20$), \alan{} only learns a program for when three buttons must be pressed ($n=3$).
On the large dataset ($p=200$), \alan{} only learns a program for when one button must be pressed ($n=1$).
By contrast, on both datasets, \name{} learns a program for when ten buttons must be pressed ($n=10$), i.e.~a program with ten body literals.
Moreover, \name{} always learns a solution comfortably within the time limit.
This result strongly suggests that the answer to \textbf{Q1} is yes, constraints can drastically improve learning performance.

\begin{figure}[ht]
\centering
\begin{tabular} {cc}
\begin{tikzpicture}[scale=0.7]
    \begin{axis}[
    xlabel=Number of pressed buttons ($n$),
    ylabel=Learning time (seconds),
    xmin=1,xmax=10,
    ymin=0,ymax=62,
    ylabel style={yshift=-6mm},
    xtick = {1,2,3,4,5,6,7,8,9,10},
    legend style={at={(0.97,0.92)},style={font=\footnotesize,nodes={right}}}
    ]

\addplot+[blue,mark=*,mark options={fill=blue},error bars/.cd,y dir=both,y explicit]
table [
x=size,
y=time,
y error plus expr=\thisrow{error},
y error minus expr=\thisrow{error},
] {data/buttonsa/popper.table};

\addplot+[gray,mark=none,dashed,error bars/.cd,y dir=both,y explicit]
table [
x=size,
y=time,
y error plus expr=\thisrow{error},
y error minus expr=\thisrow{error},
] {data/buttonsa/unconstrained.table};

\addplot+[red,mark=square*,mark options={color=red},error bars/.cd,y dir=both,y explicit]
table [
x=size,
y=time,
y error plus expr=\thisrow{error},
y error minus expr=\thisrow{error},
] {data/buttonsa/metagol.table};

\addplot+[green,mark=triangle*,mark options={color=green},error bars/.cd,y dir=both,y explicit]
table [
x=size,
y=time,
y error plus expr=\thisrow{error},
y error minus expr=\thisrow{error},
] {data/buttonsa/ilasp2.table};

\addplot+[brown,mark=x,mark options={color=brown},error bars/.cd,y dir=both,y explicit]
table [
x=size,
y=time,
y error plus expr=\thisrow{error},
y error minus expr=\thisrow{error},
] {data/buttonsa/ilasp3.table};

\addplot+[black,mark=diamond*,mark options={color=black},error bars/.cd,y dir=both,y explicit]
table [
x=size,
y=time,
y error plus expr=\thisrow{error},
y error minus expr=\thisrow{error},
] {data/buttonsa/aleph.table};

    \legend{\name{},\alan{},Metagol,ILASP2i,ILASP3,Aleph}
    \end{axis}
  \end{tikzpicture}
&
\begin{tikzpicture}[scale=0.7]
    \begin{axis}[
    xlabel=Number of pressed buttons ($n$),
    ylabel=Learning time (seconds),
    xmin=1,xmax=10,
    ymin=0,ymax=62,
    xtick = {1,2,3,4,5,6,7,8,9,10},
    ylabel style={yshift=-6mm},
    legend style={at={(0.97,0.39)},style={font=\footnotesize,font=\footnotesize,nodes={right}}}
    ]

\addplot+[blue,mark=*,mark options={fill=blue},error bars/.cd,y dir=both,y explicit]
table [
x=size,
y=time,
y error plus expr=\thisrow{error},
y error minus expr=\thisrow{error},
] {data/buttonsb/popper.table};

\addplot+[gray,mark=none,dashed,error bars/.cd,y dir=both,y explicit]
table [
x=size,
y=time,
y error plus expr=\thisrow{error},
y error minus expr=\thisrow{error},
] {data/buttonsb/unconstrained.table};

\addplot+[red,mark=square*,mark options={color=red},error bars/.cd,y dir=both,y explicit]
table [
x=size,
y=time,
y error plus expr=\thisrow{error},
y error minus expr=\thisrow{error},
] {data/buttonsb/metagol.table};

\addplot+[green,mark=triangle*,mark options={color=green},error bars/.cd,y dir=both,y explicit]
table [
x=size,
y=time,
y error plus expr=\thisrow{error},
y error minus expr=\thisrow{error},
] {data/buttonsb/ilasp2.table};

\addplot+[brown,mark=x,mark options={color=brown},error bars/.cd,y dir=both,y explicit]
table [
x=size,
y=time,
y error plus expr=\thisrow{error},
y error minus expr=\thisrow{error},
] {data/buttonsb/ilasp3.table};

\addplot+[black,mark=diamond*,mark options={color=black},error bars/.cd,y dir=both,y explicit]
table [
x=size,
y=time,
y error plus expr=\thisrow{error},
y error minus expr=\thisrow{error},
] {data/buttonsb/aleph.table};

    \end{axis}
  \end{tikzpicture} \\
Small BK (p=20) & Big BK (p=200)
\end{tabular}
\caption{
    Buttons experiment.
}
\label{fig:buttonsa}
\end{figure}
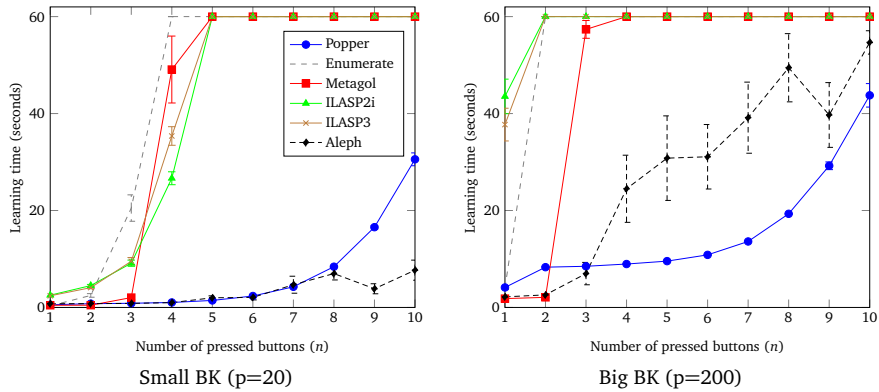

\name{} outperforms Metagol on both datasets.
For the small dataset, the largest program that Metagol learns is for when $n=4$, which takes 50 seconds to learn, compared to one second for \name{}.
For the big dataset, the largest program that Metagol learns is for when $n=3$, which takes 57 seconds to learn, compared to eight seconds for \name{}.
Metagol struggles because of its inefficient search.
Metagol performs iterative deepening over the number of clauses allowed in a solution \cite{mugg:metagold}.
However, if a clause or literal fails during the search, Metagol does not remember this failure, and will retry already failed clauses and literals at each depth (and even multiple times as the same depth).
By contrast, if a clause fails, \name{} learns constraints from the failure so it never tries that clause (or its specialisations) again.

\name{} outperforms ILASP2i and ILASP3 on both datasets.
ILASP2i only learns programs with four (small dataset) and one (big dataset) body literals.
ILASP3 only learns programs with four (small dataset) and one (big dataset) body literals.
ILASP2i and ILASP3 both struggle on this problem because they pre-compute every clause in the hypothesis space, which means that they struggle to learn clauses with many body literals.
By contrast, \name{} can learn programs with ten body literals on both datasets.

Aleph outperforms \name{} on the small dataset when $n>8$.
However, on the big dataset, \name{} outperforms Aleph when $n>3$.

Overall, the results from this experiment suggest that (i) the answer to question \textbf{Q1} is certainly yes, constraints improve learning performance, (ii) the answer to \textbf{Q2} is that \name{} scales well in terms of the number of body literals in a solution and the number of background relations, and (iii) the answer to \textbf{Q3} is that \name{} can outperform other ILP systems when varying the optimal solution size and the number of background relations.
\subsection{Robots}
\label{sec:robots}

The purpose of this second experiment is to evaluate how well \name{} scales with respect to the domain size (i.e.~the constant signature).
We therefore need a problem where we can control the domain size.
We consider a robot strategy learning problem \cite{metagolo}.
There is a robot in a $n \times n$ grid world.
Given an arbitrary start position, the goal is to learn a general strategy to move the robot to the topmost row in the grid.
For instance, for a $10 \times 10$ world and the start position $(2,2)$, the goal is to move to position $(2,10)$.
The domain contains all possible robot positions.
We therefore vary the domain size by varying $n$, the size of the world.
The optimal solution is a recursive strategy for \emph{keep moving upwards until you are at the top row}.
To reiterate, we purposely fix the optimal solution so that the only variable in the experiment is the domain size (i.e.~the grid world size), which we progressively increase to evaluate how well the systems scale.


\subsubsection{Materials}
We consider two representations: a representation for \name{}, \alan{}, Metagol, and Aleph, and then a representation designed to help ILASP solve the problem.
When given the Prolog representation, neither ILASP2i nor ILASP3 could solve any of the problems because of the grounding problem.
In both representations, we provide as BK four dyadic relations, \emph{move\_right}, \emph{move\_left}, \emph{move\_up}, and \emph{move\_down}, that change the state, e.g. \emph{move\_right((2,2),(3,2))}, and four monadic relations, \emph{at\_top}, \emph{at\_bottom}, \emph{at\_left}, and \emph{at\_right}, that check the state.
The exact language biases used can be found in Appendix \ref{app:robots}.

\paragraph{Prolog representation.}
In the Prolog representation, an example is an atom of the form $f(s_1,s_2)$, where $s_1$ and $s_2$ represent start and end states.
A state is a pair of discrete coordinates $(x,y)$ denoting the column ($x$) and row ($y$) position of the robot.

\paragraph{ILASP representation.}
When given the Prolog representation, neither ILASP2i nor ILASP3 could solve any of the problems in this experiment because of the grounding problem.
We therefore asked Mark Law to help us design a more suitable representation.
In this representation, an example is an atom of the form $f(s_2)$ where $s_2$ represents the end state.
Each example is a distinct ILASP example (a partial interpretation) with its own \emph{context}, where the start state is given in the context as \emph{start\_state($s_1$)}.
This representation alleviates the grounding problem of the Prolog representation.

\paragraph{ILASP2i and ILASP3 settings.}
We run both ILASP2i and ILASP3 with the same settings, so we again refer to both as ILASP.
We run ILASP with the `no constraints', `no aggregates', 'disable implication', 'disable propagation', and 'simple contexts' flags.
We tell ILASP that each BK relation is \emph{positive}, \emph{anti\_reflexive}, and \emph{symmetric}.
We also employ a set of `bias constraints' to reduce the hypothesis space.
We also restrict some of the recall values for the BK relations.
We set ILASP to use at most four unique variables and at most three body literals (`-ml=3' and `--max-rule-length=4').
The full language bias restrictions can be found in the appendix \ref{app:robots}.

\paragraph{Metagol settings.}
We provide Metagol with the metarules in Figure \ref{fig:metarules}.
These metarules constitute an almost\footnote{
   It is impossible to generate a finite and complete set of metarules for a singleton-free fragment of monadic and dyadic Datalog \cite{crop:reduce}.
} complete set of metarules for a singleton-free fragment of monadic and dyadic Datalog \cite{crop:reduce}.

\begin{figure}[ht]
\centering
\begin{minipage}{.3\linewidth}%
\begin{lstlisting}[frame=single]
P(A):-Q(A).
P(A):-Q(A),R(A).
P(A):-Q(A,B),R(B).
P(A):-Q(A,B),P(B).
P(A):-Q(A,B),R(A,B).
\end{lstlisting}
\end{minipage}%
\hspace{2ex}
\begin{minipage}{.3\linewidth}%
\begin{lstlisting}[frame=single]
P(A,B):-Q(B,A).
P(A,B):-Q(A,B),R(A,B).
P(A,B):-Q(A),R(A,B).
P(A,B):-Q(A,B),R(B).
P(A,B):-Q(A,C),R(C,B).
P(A,B):-Q(A,C),P(C,B).
\end{lstlisting}
\end{minipage}%
\caption{
The metarules used by Metagol in the robot and list transformation experiments.
}
\label{fig:metarules}
\end{figure}

\paragraph{Popper settings.}
We allow \name{} and \alan{} to use at most four unique variables per clause and at most three body literals (which match the ILASP settings), and at most three clauses.

\paragraph{Aleph settings.}
We set the maximum variable depth and clause length to six and set the maximum number of search nodes to 30000.

\subsubsection{Methods}
We run the experiment with an $n \times n$ grid world for each $n$ in $\{10,20,\dots,100\}$.
To generate examples, for start states, we uniformly sample positions that are not at the top of the world.
For the positive examples, the end state is the topmost position, e.g.~$(x,n)$ where $n$ is the grid size.
For negative examples, we randomly sample start and end states and reject the example if it is a positive example.
To ensure that there are some negative examples with the topmost position, in 25\% of the examples we set the end position to be the topmost row of column y, but ensure that the start position is not y.
We sample with replacement 20 positive and 20 negative training examples, and 1000 positive and 1000 negative testing examples.
The default predictive accuracy is therefore 50\%.
We measure predictive accuracies and learning times.
We enforce a timeout of one minute per task.
If a system fails to learn a solution in the given time then it only achieves default predictive accuracy (50\%).
We repeat each experiment ten times and plot the standard error.

\subsubsection{Results}
Figure \ref{fig:robots} shows the results.
\name{} achieves the best predictive accuracy out of all the systems.
\alan{} is the second best performing system, although it is does not always learn the optimal solution.
\name{} is substantially quicker than \alan{} (on average about 40 times quicker) and is the fastest of all the systems.
The learning time of \name{} slightly decreases as the grid size grows.
The reason for this is twofold.
First, when the grid world is small, there are often many small programs that cover some of the positive examples but none of the negative examples, such as:

\cpl{f(S1,S2):- move\_up(S1,S3),move\_up(S3,S2).}

\noindent
Because they cover some of the examples, Popper cannot completely rule them out.
However, as the grid size grows, these smaller programs are less likely to cover the examples because the examples are more spread out over the grid.
Second, solutions have either five or six literals,
with smaller solutions becoming more likely with increasing world size.
These reasons explain why the predictive accuracy of \alan{} improves as the grid size grows.
The reason that the learning time of \name{} does not increase is that the domain size has no influence on the size of the learning from failures hypothesis space (Proposition \ref{prop:hspace}).
The only influence the grid size has is any overhead in executing the induced Prolog program on larger grids.
This result suggests that \name{} can scale well with respect to the domain size.

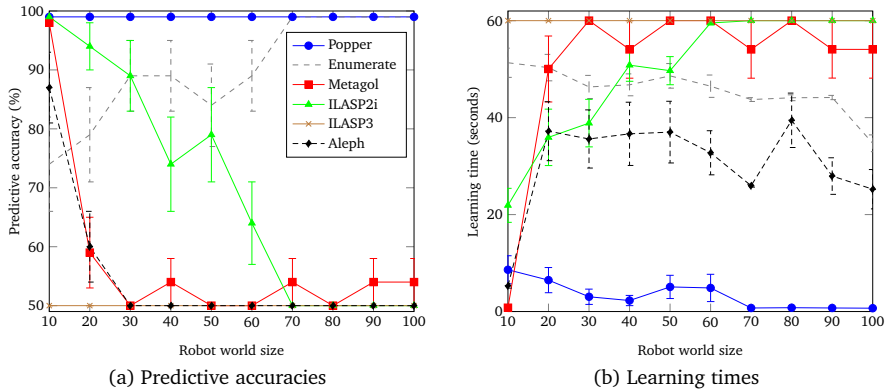
\begin{figure}[ht]
\centering
\begin{tabular} {cc}
\begin{tikzpicture}[scale=.7]
   \begin{axis}[
   xlabel=Robot world size,
   ylabel=Predictive accuracy (\%),
   xmin=10,xmax=100,
   ymin=49,ymax=100,
   xtick = {10,20,30,40,50,60,70,80,90,100},
   ylabel style={yshift=-6mm},
   legend style={at={(0.98,0.93)},style={font=\footnotesize,nodes={right}}}
   ]

\addplot+[blue,mark=*,mark options={fill=blue},error bars/.cd,y dir=both,y explicit]
table [
x=size,
y=acc,
y error plus expr=\thisrow{error},
y error minus expr=\thisrow{error},
] {data/robots/popper-prolog-acc.table};

\addplot+[gray,mark=none,dashed,error bars/.cd,y dir=both,y explicit]
table [
x=size,
y=acc,
y error plus expr=\thisrow{error},
y error minus expr=\thisrow{error},
] {data/robots/unconstrained-prolog-acc.table};

\addplot+[red,mark=square*,mark options={color=red},error bars/.cd,y dir=both,y explicit]
table [
x=size,
y=acc,
y error plus expr=\thisrow{error},
y error minus expr=\thisrow{error},
] {data/robots/metagol-acc.table};

\addplot+[green,mark=triangle*,mark options={color=green},error bars/.cd,y dir=both,y explicit]
table [
x=size,
y=acc,
y error plus expr=\thisrow{error},
y error minus expr=\thisrow{error},
] {data/robots/ilasp2-acc.table};

\addplot+[brown,mark=x,mark options={color=brown},error bars/.cd,y dir=both,y explicit]
table [
x=size,
y=acc,
y error plus expr=\thisrow{error},
y error minus expr=\thisrow{error},
] {data/robots/ilasp3-acc.table};

\addplot+[black,mark=diamond*,mark options={color=black},error bars/.cd,y dir=both,y explicit]
table [
x=size,
y=acc,
y error plus expr=\thisrow{error},
y error minus expr=\thisrow{error},
] {data/robots/aleph-acc.table};

   \legend{\name{},\alan{},Metagol,ILASP2i,ILASP3,Aleph}
   \end{axis}
 \end{tikzpicture} &
\begin{tikzpicture}[scale=.7]
   \begin{axis}[
   xlabel=Robot world size,
   ylabel=Learning time (seconds),
   xmin=10,xmax=100,
   ymin=0,ymax=62,
   xtick = {10,20,30,40,50,60,70,80,90,100},
   ylabel style={yshift=-6mm},
   legend style={legend pos=north west,style={font=\footnotesize,nodes={right}}}
   ]

\addplot+[blue,mark=*,mark options={fill=blue},error bars/.cd,y dir=both,y explicit]
table [
x=size,
y=time,
y error plus expr=\thisrow{error},
y error minus expr=\thisrow{error},
] {data/robots/popper-prolog-time.table};

\addplot+[gray,mark=none,dashed,error bars/.cd,y dir=both,y explicit]
table [
x=size,
y=time,
y error plus expr=\thisrow{error},
y error minus expr=\thisrow{error},
] {data/robots/unconstrained-prolog-time.table};

\addplot+[red,mark=square*,mark options={color=red},error bars/.cd,y dir=both,y explicit]
table [
x=size,
y=time,
y error plus expr=\thisrow{error},
y error minus expr=\thisrow{error},
] {data/robots/metagol-time.table};

\addplot+[green,mark=triangle*,mark options={color=green},error bars/.cd,y dir=both,y explicit]
table [
x=size,
y=time,
y error plus expr=\thisrow{error},
y error minus expr=\thisrow{error},
] {data/robots/ilasp2-time.table};

\addplot+[brown,mark=x,mark options={color=brown},error bars/.cd,y dir=both,y explicit]
table [
x=size,
y=time,
y error plus expr=\thisrow{error},
y error minus expr=\thisrow{error},
] {data/robots/ilasp3-time.table};

\addplot+[black,mark=diamond*,mark options={color=black},error bars/.cd,y dir=both,y explicit]
table [
x=size,
y=time,
y error plus expr=\thisrow{error},
y error minus expr=\thisrow{error},
] {data/robots/aleph-time.table};

   \end{axis}
 \end{tikzpicture} \\
(a) Predictive accuracies & (b) Learning times
\end{tabular}
\caption{
Robots experimental results when varying the world size, which corresponds to the domain size.
}
\label{fig:robots}
\end{figure}

\name{} outperforms Metagol in all cases.
For a small $10x10$ grid world, Metagol learns the optimal solution and does so quicker than \name{} (Metagol takes 1 second compared to \name{} which takes 9 seconds).
However, as the grid size grows, Metagol's performance quickly degrades.
For a grid size greater than 20, Metagol almost always times out before finding a solution.
The reason is that Metagol searches for a hypothesis by inducing and executing partial programs over the examples.
In other words, Metagol uses the examples to guide the hypothesis search.
As the grid size grows, there are more partial programs to construct, so its performance suffers.
Note that when Metagol learns a solution, it is always an accurate solution.

\name{} outperforms ILASP2i and ILASP3 both in terms of predictive accuracies and learning times.
ILASP3 cannot learn any solutions in the given time, even for the $10x10$ world.
ILASP2i initially learns solutions in the given time limit, but struggles as the grid size grows.
Note that when ILASP2i learns a solution, it is always an accurate solution.

ILASP2i outperforms ILASP3 because once ILASP2i finds a solution it terminates.
By contrast, ILASP3 finds one hypothesis schema that guarantees coverage of the example (which, in this special case, also implies finding a solution), then carries on to find alternative hypothesis schemas.
The extra work done by ILASP3 is needed when learning general ASP programs, but in this special case (where there no ILASP negative examples) it is unnecessary and computationally expensive.
We refer the reader to Law's thesis \cite{law:thesis} for a detailed comparison of ILASP2i and ILASP3\footnote{We thank Mark Law for this explanation.
}.

\name{} outperforms Aleph.
For small grid worlds, Aleph sometimes learns programs that generalise to the training set (such as move up three times).
But as the grid size grows, Aleph struggles because it struggles to learn recursive programs.


Overall, the results from this experiment suggest that (i) the answer to question \textbf{Q1} is certainly yes, constraints improve learning performance, (ii) the answer to \textbf{Q2} is that \name{} scales well in terms of the domain size, and (iii) the answer to \textbf{Q3} is that \name{} can outperform other ILP systems when varying the domain size.
\subsection{List transformation problem}
\label{sec:lists}

The purpose of this third experiment is to evaluate how well \name{} performs on difficult (mostly recursive) list transformation problems.
Learning recursive programs has long been considered a difficult problem in ILP \cite{ilp20} and most ILP and program synthesis systems cannot learn recursive programs.
Because ILASP2i and ILASP3 do not support lists, we only compare \name{}, \alan{}, Metagol, and Aleph.

\subsubsection{Materials}

We evaluate the systems on the ten list transformation tasks shown in Table \ref{tab:listprobs}.
These tasks include a mix of monadic (e.g.~\tw{evens} and \tw{sorted}), dyadic (e.g.~\tw{droplast} and \tw{finddup}), and triadic (\tw{dropk}) target predicates.
The tasks also contain a mix of functional (e.g.~\tw{last} and \tw{len}) and relational problems (e.g.~\tw{finddup} and \tw{member}).
These tasks are extremely difficult for ILP systems.
To learn solutions that generalise, an ILP system needs to support recursion and large domains.
As far as we are aware, no existing ILP system can learn optimal solutions for all of these tasks without being provided with a strong inductive bias\footnote{
As mentioned in Section \ref{sec:recursion}, some inverse entailment methods \cite{progol} might sometimes learn solutions for them.
However, these approaches need an example to learn the base case of a recursive program and then an example to learn the inductive case.
Moreover, these approaches would not be guaranteed to learn the optimal solution.
Metagol could possibly learn solutions for them if given the exact metarules needed, but that requires that you know the solution before you try to learn it.
}.

We give each system the following dyadic relations \emph{head}, \emph{tail}, \emph{decrement}, \emph{geq} and the monadic relations \emph{empty}, \emph{zero}, \emph{one}, \emph{even}, and \emph{odd}.
We also include the dyadic relation \emph{increment} in the \tw{len} experiment.
We had to remove this relation from the BK for the other experiments because when given this relation Metagol runs into infinite recursion\footnote{
    Because Metagol induces hypotheses by partially constructing and evaluating hypotheses, it is very difficult to impose a timeout on a particular hypothesis, which we can easily do with \name{}.
} on almost every problem and could not find any solutions.
We also include \emph{member/2} in the find duplicate problem.
We also include \emph{cons/3} in the \tw{addhead}, \tw{dropk}, and \tw{droplast} experiments.
We exclude this relation from the other experiments because Metagol does not easily support triadic relations.
The exact language biases used can be found in Appendix \ref{app:lists}.

\paragraph{Metagol settings.}
For Metagol, we use almost the same metarules as in the previous robot experiment (Figure \ref{fig:metarules}).
However, when given the \emph{inverse} metarule $P(A,B) \leftarrow Q(B,A)$, Metagol could not learn any solution, again because of infinite recursion.
To aid Metagol, we therefore replace the \emph{inverse} metarule with the \emph{identity} metarule, i.e.~$P(A,B) \leftarrow Q(A,B)$.
In addition, when we first ran the experiment with randomly ordered examples, we found that Metagol struggled to find solutions for all the problems (except \tw{member}).
The reason is that Metagol is sensitive to the order of examples because it uses the examples in the order they are given to induce a hypothesis.
Therefore, to aid Metagol, we provide the examples in increasing size (i.e.~the length of the input lists).

\paragraph{\name{} and \alan{} settings.}
We set \name{} and \alan{} to use at most five unique variables, at most five body literals, and at most two clauses.
In Section \ref{sec:sensitivity}, we evaluate how sensitive \name{} is to these parameters.
For each BK relation, we also provide both systems with simple types and argument directions (whether input or output).
Because \name{} and \alan{} can generate non-terminating Prolog programs, we set both systems to use a testing timeout of 0.1 seconds per example.
If a program times out, we view it as a failure.

\paragraph{Aleph settings.}
We give Aleph identical mode declarations and determinations to \name{} and \alan{}.
We set the maximum variable depth and clause length to six and set the maximum number of search nodes to 30000.

\begin{table}[ht]
\centering
\footnotesize
\lstset{basicstyle=\footnotesize}
\begin{tabular}{l|l|l}
\toprule
\textbf{Name} & \textbf{Description} & \textbf{Example solution}\\
\midrule
\tw{addhead} & Prepend the head three times&
\begin{lstlisting}
addhead(A,B):-head(A,C),cons(C,A,D),cons(C,D,E),cons(C,E,B).
\end{lstlisting}\\
\midrule
\tw{dropk} & Drop the first k elements&
\begin{lstlisting}
dropk(A,B,C):-one(B),tail(A,C).
dropk(A,B,C):-tail(A,D),decrement(B,E),dropk(D,E,C).
\end{lstlisting}\\
\midrule
\tw{droplast} & Drop the last element&
\begin{lstlisting}
droplast(A,B):-tail(A,B),empty(B).
droplast(A,B):-tail(A,C),droplast(C,D),head(A,E),cons(E,D,B).
\end{lstlisting}\\
\midrule
\tw{evens} & Check all elements are even&
\begin{lstlisting}
evens(A):-empty(A).
evens(A):-head(A,B),even(B),tail(A,C),evens(C).
\end{lstlisting}\\
\midrule
\tw{finddup} & Find duplicate elements&
\begin{lstlisting}
finddup(A,B):-head(A,B),tail(A,C),member(B,C).
finddup(A,B):-tail(A,C),finddup(C,B).
\end{lstlisting}\\
\midrule
\tw{last} & Last element&
\begin{lstlisting}
last(A,B):-tail(A,C),empty(C),head(A,B).
last(A,B):-tail(A,C),last(C,B).
\end{lstlisting}\\
\midrule
\tw{len} & Calculate list length&
\begin{lstlisting}
len(A,B):-empty(A),zero(B).
len(A,B):-tail(A,C),len(C,D),succ(D,B).
\end{lstlisting}\\
\midrule
\tw{member} & Member of a list&
\begin{lstlisting}
member(A,B):-head(A,B).
member(A,B):-tail(A,C),member(C,B).
\end{lstlisting}\\
\midrule
\tw{sorted} & Check list is sorted&
\begin{lstlisting}
sorted(A):-tail(A,B),empty(B).
sorted(A):-head(A,B),tail(A,C),head(C,D),geq(D,B),sorted(C).
\end{lstlisting}\\
\midrule
\tw{threesame} & First three elements are identical&
\begin{lstlisting}
threesame(A):-head(A,B),tail(A,C),head(C,B),tail(C,D),head(D,B).
\end{lstlisting}\\
\bottomrule
\end{tabular}
\caption{Example solutions for the list transformation problems.}
\label{tab:listprobs}
\end{table}

\subsubsection{Methods}
For each problem, we generate 10 positive and 10 negative training examples, and 1000 positive and 1000 negative testing examples.
The default predictive accuracy is therefore 50\%.
Each list is randomly generated and has a maximum length of 50.
We sample the list elements uniformly at random from the set $\{1,2,\dots,100\}$.
We measure the predictive accuracy and learning times.
We enforce a timeout of five minutes per task.
We repeat each experiment 10 times and plot the standard error.

\subsubsection{Results}
Table \ref{tab:listaccs} shows that \name{} equals or outperforms \alan{} on all the tasks in terms of predictive accuracies.
When a system has 50\% accuracy, it means that the system has failed to learn a program in the given amount of time, and so achieves the default accuracy.
Table \ref{tab:listtimes} shows that \name{} substantially outperforms \alan{} in terms of learning times.
For instance, whereas it takes \alan{} 159 seconds to find an \tw{evens} program, it takes \name{} only four seconds.
Table \ref{tab:popanalysis} decomposes the learning times of \name{}.

Table \ref{tab:listaccs} shows that \name{} equals or outperforms Metagol on all the tasks in terms of predictive accuracies, except the \tw{finddup} problem, where Metagol has a 2\% higher predictive accuracy.
Table \ref{tab:listaccs} also shows that Aleph struggles to learn solutions to these problems.
The exceptions are \tw{addhead} and \tw{threesame}, which do not need recursion.



Overall, the results from this experiment suggest that (i) the answer to question \textbf{Q1} is again yes, constraints improve learning performance, and (ii) \name{} can outperform other ILP systems when learning complex and recursive list transformation programs.

\begin{table}[ht]
\centering
\begin{tabular}{l|c|c|c|c}
\toprule
\textbf{Name} & \textbf{\name{}} & \textbf{\alan{}} & \textbf{Metagol} & \textbf{Aleph}\\
\midrule
\tw{addhead} & \textbf{100} $\pm$ 0 & \textbf{100} $\pm$ 0 & n/a & 90 $\pm$ 10 \\
\tw{dropk} & \textbf{100} $\pm$ 0 & 50 $\pm$ 0 & n/a & 50 $\pm$ 0 \\
\tw{droplast} & \textbf{100} $\pm$ 0 & 50 $\pm$ 0 & n/a & 50 $\pm$ 0 \\
\tw{evens} & \textbf{100} $\pm$ 0 & \textbf{100} $\pm$ 0 & 50 $\pm$ 0 & 50 $\pm$ 0 \\
\tw{finddup} & 98 $\pm$ 0 & 50 $\pm$ 0 & \textbf{100} $\pm$ 0 & 50 $\pm$ 0 \\
\tw{last} & \textbf{100} $\pm$ 0 & 50 $\pm$ 0 & \textbf{100} $\pm$ 0 & 50 $\pm$ 0 \\
\tw{len} & \textbf{100} $\pm$ 0 & 50 $\pm$ 0 & 50 $\pm$ 0 & 50 $\pm$ 0 \\
\tw{member} & \textbf{100} $\pm$ 0 & \textbf{100} $\pm$ 0 & \textbf{100} $\pm$ 0 & 50 $\pm$ 0 \\
\tw{sorted} & \textbf{100} $\pm$ 0 & 50 $\pm$ 0 & 50 $\pm$ 0 & 68 $\pm$ 2 \\
\tw{threesame} & \textbf{99} $\pm$ 0 & \textbf{99} $\pm$ 0 & \textbf{99} $\pm$ 0 & \textbf{99} $\pm$ 0 \\
\bottomrule
\end{tabular}
\caption{
List transformation predictive accuracies.
We round accuracies to integer values.
The error is standard error.
}
\label{tab:listaccs}
\end{table}

\begin{table}[ht]
\centering
\begin{tabular}{l|c|c|c|c}
\toprule
\textbf{Name} & \textbf{\name{}} & \textbf{\alan{}} & \textbf{Metagol} & \textbf{Aleph}\\
\midrule
\tw{addhead} & \textbf{0.5} $\pm$ 0 & 2 $\pm$ 0 & n/a & 103 $\pm$ 49 \\
\tw{dropk} & \textbf{0.8} $\pm$ 0 & 300 $\pm$ 0 & n/a & 3 $\pm$ 0.2 \\
\tw{droplast} & \textbf{3} $\pm$ 0.1 & 300 $\pm$ 0 & n/a & 300 $\pm$ 0 \\
\tw{evens} & 4 $\pm$ 0.1 & 159 $\pm$ 0.1 & 300 $\pm$ 0 & \textbf{1} $\pm$ 0 \\
\tw{finddup} & 36 $\pm$ 2 & 300 $\pm$ 0 & 2 $\pm$ 0.5 & \textbf{1} $\pm$ 0.1 \\
\tw{last} & 2 $\pm$ 0.1 & 300 $\pm$ 0 & \textbf{0.7} $\pm$ 0.2 & 1 $\pm$ 0.1 \\
\tw{len} & 12 $\pm$ 0.3 & 300 $\pm$ 0 & 300 $\pm$ 0 & \textbf{1} $\pm$ 0 \\
\tw{member} & 0.4 $\pm$ 0.1 & 7 $\pm$ 0 & \textbf{0.3} $\pm$ 0 & 0.9 $\pm$ 0.1 \\
\tw{sorted} & 23 $\pm$ 1 & 300 $\pm$ 0 & 300 $\pm$ 0 & \textbf{0.8} $\pm$ 0 \\
\tw{threesame} & \textbf{0.2} $\pm$ 0.1 & 0.4 $\pm$ 0.2 & 0.9 $\pm$ 0.3 & 0.5 $\pm$ 0 \\
\bottomrule
\end{tabular}
\caption{
List transformation learning times.
We round times over 1 second to the nearest second.
The error is standard error.
Note that although Aleph is sometimes faster than \name{}, it only learns accurate solutions for \tw{addhead} and \tw{threesame}.
}
\label{tab:listtimes}
\end{table}

\begin{table}[ht]
\centering
\begin{tabular}{l|c|c|c}
\toprule
\textbf{Name} & \textbf{Time} & \textbf{Grounding} & \textbf{Solving}\\
\midrule
\tw{addhead} & 0.5 $\pm$ 0 & 0.1 $\pm$ 0 & 0.2 $\pm$ 0 \\
\tw{dropk} & 0.8 $\pm$ 0 & 0.3 $\pm$ 0 & 0.1 $\pm$ 0 \\
\tw{droplast} & 3 $\pm$ 0.1 & 0.4 $\pm$ 0.1 & 1 $\pm$ 0 \\
\tw{evens} & 4 $\pm$ 0.1 & 1 $\pm$ 0 & 1 $\pm$ 0.1 \\
\tw{finddup} & 36 $\pm$ 2 & 25 $\pm$ 1 & 7 $\pm$ 0.5 \\
\tw{last} & 2 $\pm$ 0.1 & 1 $\pm$ 0 & 0.5 $\pm$ 0 \\
\tw{len} & 12 $\pm$ 0.3 & 7 $\pm$ 0.2 & 2 $\pm$ 0.1 \\
\tw{member} & 0.4 $\pm$ 0.1 & 0.1 $\pm$ 0 & 0.1 $\pm$ 0 \\
\tw{sorted} & 23 $\pm$ 1 & 12 $\pm$ 0.9 & 8 $\pm$ 0.6 \\
\tw{threesame} & 0.2 $\pm$ 0.1 & 0 $\pm$ 0 & 0 $\pm$ 0 \\
\bottomrule
\end{tabular}
\caption{
Decomposition of \name{} learning times.
The unaccounted time (time not grounding or solving) is mostly the overhead of testing the induced Prolog programs.
}
\label{tab:popanalysis}
\end{table}


\subsection{Scalability}
\label{sec:scalability}

Our buttons experiment (Experiment \ref{sec:buttons}) showed that \name{} scales well in the size of the optimal solution size and the number of background relations.
Our robot experiment (Experiment \ref{sec:robots}) showed that \name{} scales well in the size of the domain.
The purpose of this experiment is to evaluate how well \name{} scales in terms of the (i) number of examples and (ii) the size of examples.
To do so, we repeat the \tw{last} experiment from Section \ref{sec:lists}, where \name{} and Metagol achieved similar performance.

\subsubsection{Materials}
We use the same materials as Section \ref{sec:lists}.

\subsubsection{Settings}
We run two experiments.
In the first experiment we vary the number of examples.
In the second experiment we vary the size of the examples (the size of the input list).
For each experiment, we measure the predictive accuracy and learning times averaged over 10 repetitions.

\paragraph{Number of examples.}
For each n in $\{1000,2000,\dots,10000\}$, we generate $n$ positive and $n$ negative training examples, and 1000 positive and 1000 negative testing examples and each element is a random integer from the range 1 to 1000.

\paragraph{Example size.}
For each $s$ in $\{50,100,150,\dots,500\}$, we generate 10 positive and 10 negative training examples, and 1000 positive and 1000 negative testing examples, where each list is of length $s$ and each element is a random integer from the range 1 to 1000.

\subsubsection{Results}
Figure \ref{fig:lists-numexs} shows the results when varying the number of training examples.
The predictive accuracies of \name{} and Metagol are almost identical until around 10,000 examples.
Given this many examples, Metagol struggles to find a solution in one minute and eventually converges on the default predictive accuracy (50\%).
By contrast, \name{} does not struggle to find a solution, even given 20,000 examples.
Figure \ref{fig:lists-numexs} shows the learning times of both systems.
The learning time of \name{} increases linearly simply because of the overhead of testing hypotheses on more examples.
The results from this experiment suggest that the answer to \textbf{Q2} is that \name{} scales well with respect the number of examples.

\begin{figure}[ht]
\centering
\begin{tabular} {cc}
\pgfplotsset{scaled x ticks=false}
\begin{tikzpicture}[scale=.7]
    \begin{axis}[
    xlabel=Number of examples,
    ylabel=Predictive accuracy (\%),
    xmin=2000,xmax=20000,
    ymin=48,ymax=101,
    ylabel style={yshift=-6mm},
    legend style={legend pos=south west,style={font=\footnotesize,nodes={right}}}
    ]

\addplot+[blue,mark=*,mark options={fill=blue},error bars/.cd,y dir=both,y explicit]
table [
x=size,
y=acc,
y error plus expr=\thisrow{error},
y error minus expr=\thisrow{error},
] {data/lists-numexs/popper-acc.table};

\addplot+[gray,mark=square*,mark options={color=gray},error bars/.cd,y dir=both,y explicit]
table [
x=size,
y=acc,
y error plus expr=\thisrow{error},
y error minus expr=\thisrow{error},
] {data/lists-numexs/metagol-acc.table};

    \legend{\name{},Metagol}
    \end{axis}
  \end{tikzpicture} &
\pgfplotsset{scaled x ticks=false}
\begin{tikzpicture}[scale=.7]
    \begin{axis}[
    xlabel=Number of examples,
    ylabel=Learning time (seconds),
    xmin=2000,xmax=20000,
    ymin=0,ymax=61,
    ylabel style={yshift=-5mm},
    legend style={legend pos=north west,style={font=\footnotesize,nodes={right}}}
    ]

\addplot+[blue,mark=*,mark options={fill=blue},error bars/.cd,y dir=both,y explicit]
table [
x=size,
y=time,
y error plus expr=\thisrow{error},
y error minus expr=\thisrow{error},
] {data/lists-numexs/popper-time.table};

\addplot+[gray,mark=square*,mark options={color=gray},error bars/.cd,y dir=both,y explicit]
table [
x=size,
y=time,
y error plus expr=\thisrow{error},
y error minus expr=\thisrow{error},
] {data/lists-numexs/metagol-time.table};
    \legend{\name{},Metagol}
    \end{axis}
  \end{tikzpicture} \\
(a) Predictive accuracies & (b) Learning times
\end{tabular}
\caption{The experimental results for the \tw{last} task when varying the number of training examples.}
\label{fig:lists-numexs}
\end{figure}
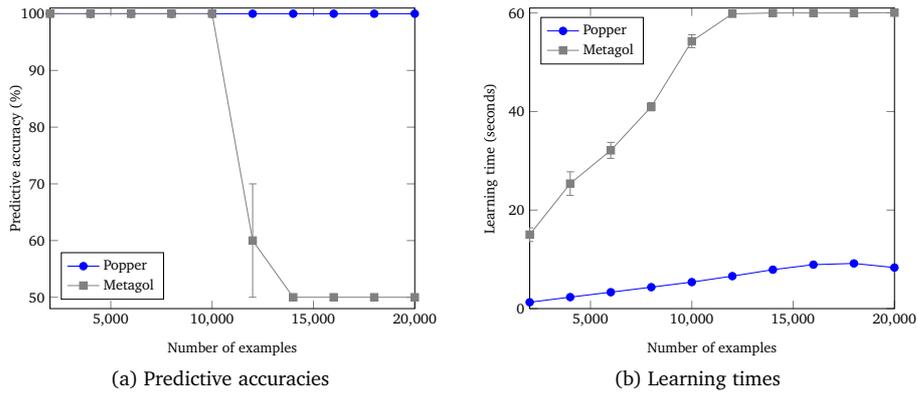

Figure \ref{fig:lists-size} shows the results when varying the size of the input (i.e. the size of the input list).
\name{} outperforms Metagol in all cases.
The mean learning times of \name{} for examples of length 50 and 500 are both less than a second.
The reason is that \name{} only uses the examples to test a hypothesis, so any increase in running time simply comes from executing the hypotheses using Prolog.
By contrast, Metagol's performance drastically degrades as the size of the examples grows.
The mean learning times for Metagol for examples of length 50 and 500 are 20 and 54 seconds respectively.
The reason is that Metagol uses the examples to search for a hypothesis by inducing and executing partial programs over the examples.
These results suggest that the answer to \textbf{Q3} is yes and the answer to \textbf{Q2} is that \name{} scales well with respect to the size of examples.

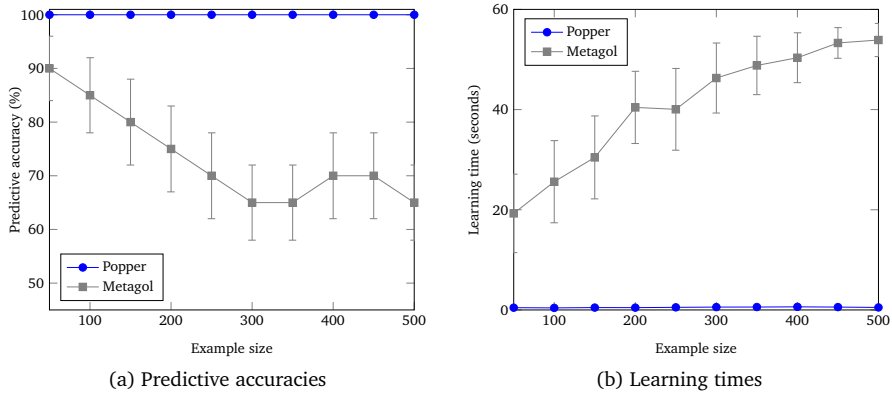
\begin{figure}[ht]
\centering
\begin{tabular} {cc}
\begin{tikzpicture}[scale=.7]
    \begin{axis}[
    xlabel=Example size,
    ylabel=Predictive accuracy (\%),
    xmin=50,xmax=500,
    ymin=45,ymax=101,
    ylabel style={yshift=-6mm},
    legend style={legend pos=south west,style={font=\footnotesize,nodes={right}}}
    ]

\addplot+[blue,mark=*,mark options={fill=blue},error bars/.cd,y dir=both,y explicit]
table [
x=size,
y=acc,
y error plus expr=\thisrow{error},
y error minus expr=\thisrow{error},
] {data/lists-size/popper-acc.table};

\addplot+[gray,mark=square*,mark options={color=gray},error bars/.cd,y dir=both,y explicit]
table [
x=size,
y=acc,
y error plus expr=\thisrow{error},
y error minus expr=\thisrow{error},
] {data/lists-size/metagol-acc.table};

    \legend{\name{},Metagol}
    \end{axis}
  \end{tikzpicture} &
\begin{tikzpicture}[scale=.7]
    \begin{axis}[
    xlabel=Example size,
    ylabel=Learning time (seconds),
    xmin=50,xmax=500,
    ymin=0,ymax=60,
    ylabel style={yshift=-5mm},
    legend style={legend pos=north west,style={font=\footnotesize,nodes={right}}}
    ]

\addplot+[blue,mark=*,mark options={fill=blue},error bars/.cd,y dir=both,y explicit]
table [
x=size,
y=time,
y error plus expr=\thisrow{error},
y error minus expr=\thisrow{error},
] {data/lists-size/popper-time.table};

\addplot+[gray,mark=square*,mark options={color=gray},error bars/.cd,y dir=both,y explicit]
table [
x=size,
y=time,
y error plus expr=\thisrow{error},
y error minus expr=\thisrow{error},
] {data/lists-size/metagol-time.table};

    \legend{\name{},Metagol}
    \end{axis}
  \end{tikzpicture} \\
(a) Predictive accuracies & (b) Learning times
\end{tabular}
\caption{The experimental results for the \tw{last} task when varying the size (list length) of training examples.}
\label{fig:lists-size}
\end{figure}
\subsection{Sensitivity}
\label{sec:sensitivity}

The learning from failures hypothesis space (Proposition \ref{prop:hspace}) is a function of the number of predicate declarations and three other variables:
\begin{itemize}
    \setlength\itemsep{0pt}
    \setlength\parskip{0pt}
    \item the maximum number of unique variables in a clause
    \item the maximum number of body literals allowed in a clause
    \item the maximum number of clauses allowed in a hypothesis
\end{itemize}

\noindent
The purpose of this experiment is to evaluate how sensitive \name{} is to these parameters.
To do so, we repeat the \tw{len} experiment from Section \ref{sec:lists} with the same BK, settings, and method, except we run three separate experiments where we vary the three aforementioned parameters.

\subsubsection{Results}

Figure \ref{fig:sens} shows the experimental results.
The results show that \name{} is sensitive to the maximum number of unique variables, which has a strong influence on learning times.
This result follows from Proposition \ref{prop:hspace} because more variables implies more ways to form literals in a clause.
Somewhat surprisingly, doubling the number of variables from 4 to 8 has little difference on performance, which suggests that \name{} is robust to imperfect parameters.
The results show that \name{} is mostly insensitive to the maximum number of body literals in a clause.
The main reason is that \name{} does not pre-compute every possible clause in the hypothesis space, which is, for instance, the case with ILASP2i and ILASP3.
The results show that \name{} scales linearly with the maximum number of clauses.
Overall these results suggest that \name{} scales well with the maximum number of body literals, but can struggle with very large values for the maximum number of unique variables and clauses.

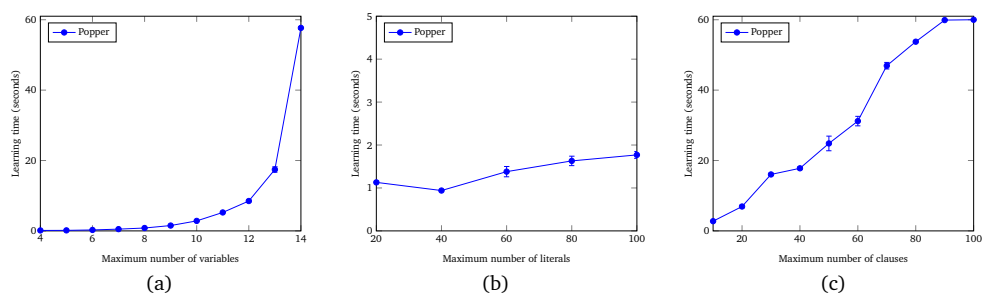
\begin{figure}[ht]
\centering
\begin{tabular} {ccc}
\begin{tikzpicture}[scale=.5]
    \begin{axis}[
    xlabel=Maximum number of variables,
    ylabel=Learning time (seconds),
    xmin=4,xmax=14,
    ymin=0,ymax=61,
    ylabel style={yshift=-6mm},
    legend style={legend pos=north west,style={font=\footnotesize,nodes={right}}}
    ]

\addplot+[blue,mark=*,mark options={fill=blue},error bars/.cd,y dir=both,y explicit]
table [
x=size,
y=time,
y error plus expr=\thisrow{error},
y error minus expr=\thisrow{error},
] {data/len/vars.table};
    \legend{\name{},Metagol}
    \end{axis}
  \end{tikzpicture} &
\pgfplotsset{scaled x ticks=false}
\begin{tikzpicture}[scale=.5]
    \begin{axis}[
    xlabel=Maximum number of literals,
    ylabel=Learning time (seconds),
    xmin=20,xmax=100,
    ymin=0,ymax=5,
    ylabel style={yshift=-7mm},
    legend style={legend pos=north west,style={font=\footnotesize,nodes={right}}}
    ]

\addplot+[blue,mark=*,mark options={fill=blue},error bars/.cd,y dir=both,y explicit]
table [
x=size,
y=time,
y error plus expr=\thisrow{error},
y error minus expr=\thisrow{error},
] {data/len/literals.table};

    \legend{\name{},Metagol}
    \end{axis}
  \end{tikzpicture} &
\begin{tikzpicture}[scale=.5]
    \begin{axis}[
    xlabel=Maximum number of clauses,
    ylabel=Learning time (seconds),
    ymin=0,ymax=61,
    xmin=10,xmax=100,
    ylabel style={yshift=-6mm},
    legend style={legend pos=north west,style={font=\footnotesize,nodes={right}}}
    ]

\addplot+[blue,mark=*,mark options={fill=blue},error bars/.cd,y dir=both,y explicit]
table [
x=size,
y=time,
y error plus expr=\thisrow{error},
y error minus expr=\thisrow{error},
] {data/len/clauses.table};

    \legend{\name{},Metagol}
    \end{axis}
  \end{tikzpicture} \\
(a) & (b) & (c)
\end{tabular}
\caption{
The experimental results for the \tw{len} task when varying the maximum number of unique variables (a), maximum body literals in a clause (b), and maximum number of clauses (c).
}
\label{fig:sens}
\end{figure}

\section{Conclusions and limitations}
\label{conc}

We have described an ILP approach called \emph{learning from failures} which decomposes the ILP problem into three separate stages: \emph{generate}, \emph{test}, and \emph{constrain}.
In the generate stage, the learner generates a hypothesis that satisfies a set of \emph{hypothesis constraints} (Definition \ref{def:hconstraint}).
In the test stage, the learner tests a hypothesis against training examples.
If a hypothesis fails, then, in the constrain stage, the learner learns hypothesis constraints from the failed hypothesis to prune the hypothesis space, i.e.~to constrain subsequent hypothesis generation.
In Section \ref{sec:failures}, we introduced three types of constraints based on theta-subsumption: \emph{generalisation}, \emph{specialisation}, and \emph{elimination} and proved their soundness in that they do not prune optimal solutions (Definition \ref{def:opthyp}).
This loop repeats until either (i) the learner finds an optimal solution, or (ii) there are no more hypotheses to test.
We implemented this approach in \name{}, an ILP system that learns definite programs.
\name{} combines ASP and Prolog to support types, learning optimal solutions, learning recursive programs, reasoning about lists and infinite domains, and hypothesis constraints.
We evaluated our approach on three diverse domains (toy game problems, robot strategies, and list transformations).
Our experimental results show that (i) constraints drastically reduce the hypothesis space, (ii) \name{} scales well with respect to the optimal solution size, the number of background relations, the domain size, the number of training examples, and the size of the training examples, and (iii) \name{} can substantially outperform existing ILP systems both in terms of predictive accuracies and learning times.

\subsection{Limitations and future work}
\label{sec:futurework}

\name{}, as implemented in this paper, has several limitations that future work should address.

\subsubsection{Features}

\paragraph{Non-observational predicate learning.}

Unlike some ILP systems \cite{progol,oled}, \name{} does not support non-observational predicate learning (non-OPL) \cite{progol}, where examples of the target predicates are not directly given.
Future work should address this limitation.


\paragraph{Predicate invention.}
Predicate invention has been shown to help reduce the size of target programs, which in turns reduces sample complexity and improves predictive accuracy \cite{playgol,seb:alps}.
\name{} does not currently support predicate invention.
There are two straightforward ways to support predicate invention.
\name{} could mimic Metagol by imposing metarules to restrict the form of clauses in a hypothesis and to guide the invention of new predicate symbols -- which is easy to do because, as we show in Appendix \ref{app:metarules}, \name{} can simulate metarules through hypothesis constraints.
Alternatively \name{} could mimic ILASP by supporting \emph{prescriptive} predicate invention \cite{law:thesis}, where the arity and (in ILASP's case, argument types) are pre-specified by the user.
Most of the results in this paper should extend to both approaches.

\paragraph{Negation.}
\name{} learns definite programs and tests them using Prolog.
\name{} can also trivially learn Datalog programs and test them using ASP.
In future work, we want to consider learning other types of programs.
For instance, most of our pruning techniques (except the elimination constraint) should extend to learning non-monotonic programs, such as Datalog with stratified negation.

\paragraph{Noise.}
\label{sec:fixnoise}
Most ILP systems handle noisy (misclassified) examples (Table \ref{tab:diffs}).
\name{} does not currently support noisy examples.
We can address this issue by relaxing when to apply learned hypothesis constraints and by maintaining the best hypotheses tested during the learning, i.e.~the hypothesis which entails the most positive and the fewest negative examples.
However, noise handling will likely increase learning times and future work should explore this trade-off.

\subsubsection{Better search}
An advantage of decomposing the learning problem is that it allows for a variety of algorithms and implementations, where each stage can be improved independently of the others.
For instance, any improvement to the \name{} ASP encoding that generates programs would have a major influence on learning times because it would reduce the number of programs to test.
Likewise, we can also optimise the testing step.
Future work should consider better search techniques.

\subsubsection{Better constraints}
Hypothesis constraints are central to our idea.
\name{} uses both predefined and learned constraints to improve performance.
\name{} uses predefined constraints to prune redundant programs from the hypothesis space (Section \ref{sec:generate}), such as recursive programs without a base case and subsumption redundant program.
\name{} also learns constraints from failures.
We think the most promising direction for future work is to improve both types of constraints (predefined and learned).

\paragraph{Types.}
Like many ILP systems \cite{progol,tilde,aleph,ilasp,dilp}, \name{} supports simple types to prune the hypothesis space.
However, more complex types, such as polymorphic types, can achieve better pruning for programs over structured data \cite{morel:typed}.
For instance, polymorphic types would allow us to distinguish between using a predicate on a list of integers and on a list of characters.
Refinement types \cite{DBLP:conf/pldi/PolikarpovaKS16}, i.e.~types annotated with restricting predicates, could allow a user to specify stronger program properties (other than examples), such as requiring that a reverse program provably has the property that the lengths of the input and output are the same.
In future work we want to explore whether we can express such complex types as hypothesis constraints.

\paragraph{Learned constraints.}
The constraints described in Section \ref{sec:failures} prune specialisations and generalisations of a failed hypothesis.
However, we have only briefly analysed the properties of these constraints.
We showed that these constraints are \emph{sound} (Propositions \ref{prop:generalisation_soundness} and \ref{prop:specialisation_soundness}) in that they do not prune optimal solutions.
We have not, however, considered their \emph{completeness}, in that they prune all non-optimal solutions.
Indeed, our \emph{elimination constraint}, for the special case of separable definite programs, prunes hypotheses that the generalisation and specialisation constraints miss.
In other words, the theory regarding which constraints to use is yet to be developed, and there may be many more constraints to be learned from failed hypotheses, all of which should drastically improve learning performance.
By contrast, refinement operators for clauses \cite{mis,claudien,ilp:book} and theories \cite{ilp:book,midelfart,DBLP:conf/ilp/Badea01} have been studied in detail in ILP.
Therefore, we think that this paper opens a new direction of research into identifying and analysing different constraints that we can learn from failed hypotheses.

\section*{Acknowledgements}
We foremost thank Mark Law for all of his help in writing this paper, including finding suitable ILASP representations for the experiments, for answering our many questions on the ILASP systems, and for suggesting much of the text on ILASP3.
We thank Tobias Kaminski, Sebastijan Dumančić, and Richard Evans for extremely valuable feedback on the paper.
We also thank one of the anonymous reviewers for suggesting a much more efficient constraint encoding that reduced \name{}'s learning times in the experiments by almost a half.

\bibliographystyle{plain}
\bibliography{ourbib15}

\appendix

\section{Popper metarule theory constraints}
\label{app:metarules}

\subsection{Metarules}

Let $M$ be an arbitrary metarule, i.e.~a second-order Horn clause which quantifies over predicate symbols.
For example, $\tw{P(A,B):-Q(A,C),R(C,B)}$ is known as the chain metarule.
All letters are quantified variables, with $\tw{P}$, $\tw{Q}$, and $\tw{R}$ being second-order, i.e.~needing to be substituted for by predicate symbols.

\subsection{From a metarule to literals}

Let $M =\tw{head}\tw{:-} \tw{body}_1,\ldots,\tw{body}_m$ be a metarule.
We use the clause encoding function $\mathit{encodeSizedClause}$ from section \ref{sec:encoding_clauses} to derive an encoding of a metarule.

%

\begin{example}
Consider $M = \tw{P(A,B):-Q(A,C),R(C,B)}$.
Its encoding, $\mathit{encodeSizedClause}(\code{Clause},M)$, is
\[
\begin{array}{l}
\code{head\_literal(Clause,P,2,(V0,V1))}
\code{,}\\
\code{body\_literal(Clause,Q,2,(V0,V2))}
\code{,}
\code{body\_literal(Clause,R,2,(V2,V1))}
\code{,}\\
\code{V0!=V1}
\code{,}
\code{V0!=V2}
\code{,}
\code{V1!=V2}
\code{,}
\code{clause\_size(Clause,2)}
\end{array}
\]

\end{example}

\subsection{Asserting metarule conformance}

Let $\mathit{Ms}$ be a set of metarules.
For each clause of a metarule conformant program, the clause must be an \emph{instance} of one of the metarules in $\mathit{Ms}$.
A clause $C$ is an instance of metarule $M \in Ms$ if there exists substitution $\theta$ such that $M\theta = C$.

We introduce two rules to ensure every clause of a generated program is an instance of at least one metarule.
The first rule identifies when there exists some metarule for which the clause is an instance.
The second rule is a constraint expressing that every clause of a program must be identified as being an instance of at least one metarule.

\noindent
For each $M \in \mathit{Ms}$, generate the following rule of the first kind:
\[
\begin{array}{l}
    \code{meta\_clause(Clause)}
    \code{:-}
    \mathit{encodeSizedClause}(\code{Clause},M)
    \code{.}
\end{array}
\]

\noindent
The second rule is:
\[
\begin{array}{l}
    \code{:-}
    \code{clause(Clause)}
    \code{,}
    \code{not meta\_clause(Clause).}
\end{array}
\]

\section{Language biases in buttons experiment}
\label{app:buttons}

\subsection{ILASP2i and ILASP3}

\begin{lstlisting}
#modeh(1,f, (positive)).
#modeb(1,button1, (positive)).
...
#modeb(1,button20, (positive)).
\end{lstlisting}

\subsection{\name{} and \alan{}}
\begin{lstlisting}
max_vars(1).
max_clauses(1).
head_pred(f,1).
body_pred(button1,1).
...
body_pred(button20,1).
\end{lstlisting}

\subsection{Aleph}
\begin{lstlisting}
:- aleph_set(i,6).
:- aleph_set(clauselength,2).
:- aleph_set(nodes,5000).
:- modeh(*,f(+var)).\\
:- modeb(*,button1(+var)).\\
:- determination(f/1,button1/1).\\
:- modeb(*,button2(+var)).\\
...\\
:- determination(f/1,button20/1).
\end{lstlisting}

\subsection{Metagol}
\begin{lstlisting}
metarule(unary1, [P,Q], [P,A], [[Q,A]]).
metarule(unary2, [P,Q,R], [P,A], [[Q,A],[R,A]]).
body_pred(button1/1).
...
body_pred(button20/1).
\end{lstlisting}
\section{Language biases in robots experiment}
\label{app:robots}

\subsection{ILASP2i and ILASP3}

\begin{lstlisting}
#modeh(f(var(state)), (positive)).
#modeh(start_state(var(state)), (positive)).
#modeb(3,move_up(var(state),var(state)), (anti_reflexive,symmetric,positive)).
#modeb(3,move_down(var(state),var(state)), (anti_reflexive,symmetric,positive)).
#modeb(3,move_left(var(state),var(state)), (anti_reflexive,symmetric,positive)).
#modeb(3,move_right(var(state),var(state)), (anti_reflexive,symmetric,positive)).
#modeb(3,at_top(var(state)), (positive)).
#modeb(3,at_bottom(var(state)), (positive)).
#modeb(3,at_left(var(state)), (positive)).
#modeb(3,at_right(var(state)), (positive)).
#modeb(1,start_state(var(state)), (positive)).

#bias(":- occurs(V, X), #false : occurs(V, Y), Y != X.").
#bias("occurs(X, f(X)) :- head(f(X)).").
#bias("occurs(X, start_state(X)) :- head(start_state(X)).").
#bias("occurs(X, start_state(X)) :- body(start_state(X)).").
#bias("occurs(X, at_top(X)) :- body(at_top(X)).").
#bias("occurs(X, at_bottom(X)) :- body(at_bottom(X)).").
#bias("occurs(X, at_left(X)) :- body(at_left(X)).").
#bias("occurs(X, at_right(X)) :- body(at_right(X)).").
#bias("occurs(X, move_up(X, Y)) :- body(move_up(X, Y)).").
#bias("occurs(X, move_left(X, Y)) :- body(move_left(X, Y)).").
#bias("occurs(X, move_right(X, Y)) :- body(move_right(X, Y)).").
#bias("occurs(X, move_down(X, Y)) :- body(move_down(X, Y)).").
#bias("occurs(X, move_up(Y, X)) :- body(move_up(Y, X)).").
#bias("occurs(X, move_left(Y, X)) :- body(move_left(Y, X)).").
#bias("occurs(X, move_right(Y, X)) :- body(move_right(Y, X)).").
#bias("occurs(X, move_down(Y, X)) :- body(move_down(Y, X)).").
\end{lstlisting}

\subsection{\name{} and \alan{}}
\begin{lstlisting}
max_vars(4).
max_body(3).
max_clauses(3).
head_pred(f,2).
body_pred(f,2).
body_pred(at_top,1).
body_pred(at_bottom,1).
body_pred(at_left,1).
body_pred(at_right,1).
body_pred(move_left,2).
body_pred(move_right,2).
body_pred(move_up,2).
body_pred(move_down,2).
direction(f,0,in).
direction(f,1,out).
direction(move_left,0,in).
direction(move_right,0,in).
direction(move_up,0,in).
direction(move_down,0,in).
direction(move_left,1,out).
direction(move_right,1,out).
direction(move_up,1,out).
direction(move_down,1,out).
direction(at_top,0,in).
direction(at_bottom,0,in).
direction(at_left,0,in).
direction(at_right,0,in).
\end{lstlisting}

\subsection{Aleph}
\begin{lstlisting}
:- aleph_set(i,6).
:- aleph_set(clauselength,6).
:- aleph_set(nodes,50000).
:- modeh(*,f(+state,-state)).\\
:- modeb(*,move\_up(+state,-state)).\\
:- modeb(*,move\_down(+state,-state)).\\
:- modeb(*,move\_left(+state,-state)).\\
:- modeb(*,move\_right(+state,-state)).\\
:- modeb(*,at\_top(+state)).\\
:- modeb(*,at\_bottom(+state)).\\
:- modeb(*,at\_left(+state)).\\
:- modeb(*,at\_right(+state)).\\
:- determination(f/2,move\_up/2).\\
:- determination(f/2,move\_down/2).\\
:- determination(f/2,move\_left/2).\\
:- determination(f/2,move\_right/2).\\
:- determination(f/2,at\_top/1).\\
:- determination(f/2,at\_bottom/1).\\
:- determination(f/2,at\_left/1).\\
:- determination(f/2,at\_right/1).\\
\end{lstlisting}

\subsection{Metagol}
\begin{lstlisting}
body_pred(move_right/2).
body_pred(move_left/2).
body_pred(move_up/2).
body_pred(move_down/2).
body_pred(at_top/1).
body_pred(at_bottom/1).
body_pred(at_left/1).
body_pred(at_right/1).
metarule([P,Q], [P,A], [[Q,A]]).
metarule([P,Q], [P,A], [[Q,A]]).
metarule([P,Q,R], [P,A], [[Q,A,B],[R,B]]).
metarule([P,Q], [P,A], [[Q,A,B],[P,B]]).
metarule([P,Q,R], [P,A], [[Q,A,B],[R,A,B]]).
metarule([P,Q], [P,A,B], [[Q,A,B]]).
metarule([P,Q,R], [P,A,B], [[Q,A,B],[R,A,B]]).
metarule([P,Q,R], [P,A,B], [[Q,A],[R,A,B]]).
metarule([P,Q,R], [P,A,B], [[Q,A,B],[R,B]]).
metarule([P,Q,R], [P,A,B], [[Q,A,C],[R,C,B]]).
metarule([P,Q], [P,A,B], [[Q,A,C],[P,C,B]]).
metarule([P,Q], [P,A,B], [[Q,B,A]]).
\end{lstlisting}
\section{Language biases in lists experiment}
\label{app:lists}

\subsection{\name{} and \alan{}}
For each list transformation problem, we have a specific bias to specify the target relations, such as the following bias for the \tw{finddup} problem:

\begin{lstlisting}
head_pred(f,2).
type(f,0,list).
type(f,1,element).
direction(f,0,in).
direction(f,1,out).
body_pred(f,2).
\end{lstlisting}

\noindent
For all the problems we include the following biases:

\begin{lstlisting}
max_vars(5).
max_body(5).
max_clauses(2).
body_pred(head,2).
body_pred(tail,2).
body_pred(geq,2).
body_pred(empty,1).
body_pred(even,1).
body_pred(odd,1).
body_pred(one,1).
body_pred(zero,1).
body_pred(decrement,2).
body_pred(increment,2). % ONLY FOR SORTED
body_pred(element,2). % ONLY FOR FIND DUPLICATE
body_pred(cons,2). % ONLY FOR ADDHEAD, DROPK, DROPLAST
type(cons,0,element).
type(cons,1,list).
type(cons,2,list).
direction(cons,0,in).
direction(cons,1,in).
direction(cons,2,out).
type(head,0,list).
type(head,1,element).
direction(head,0,in).
direction(head,1,out).
type(tail,0,list).
type(tail,1,list).
direction(tail,0,in).
direction(tail,1,out).
type(empty,0,list).
direction(empty,0,in).
type(element,0,list).
type(element,1,element).
direction(element,0,in).
direction(element,1,out).
type(increment,0,element).
type(increment,1,element).
direction(increment,0,in).
direction(increment,1,out).
type(decrement,0,element).
type(decrement,1,element).
direction(decrement,0,in).
direction(decrement,1,out).
type(geq,0,element).
type(geq,1,element).
direction(geq,0,in).
direction(geq,1,in).
type(even,0,element).
direction(even,0,in).
type(odd,0,element).
direction(odd,0,in).
type(one,0,element).
direction(one,0,in).
type(zero,0,element).
direction(zero,0,out).
\end{lstlisting}

\subsection{Aleph}
For each list transformation problem, we have a specific bias to specify the target relations, such as the following bias for the \tw{finddup} problem:

\begin{lstlisting}
:- modeh(*,f(+list,-element)).\\
:- modeb(*,f(+list,-element)).\\
\end{lstlisting}

\noindent
For all the problems we include the following biases (we we replace \tw{f/2} in the determinations with the correct arity of the target predicate):

Note that \tw{increment} is only given in the \tw{sorted} experiment, \tw{element} is only given in the \tw{finddupli} experiment, and \tw{cons} is only given in the \tw{addhead}, \tw{dropk}, and \tw{droplast} experiments.

\begin{lstlisting}
:- aleph_set(i,6).\\
:- aleph_set(clauselength,6).\\
:- aleph_set(nodes,30000).\\
:- modeb(*,head(+list,-element)).\\
:- modeb(*,tail(+list,-list)).\\
:- modeb(*,geq(+element,+element)).\\
:- modeb(*,empty(+list)).\\
:- modeb(*,even(+element)).\\
:- modeb(*,odd(+element)).\\
:- modeb(*,one(+element)).\\
:- modeb(*,zero(-element)).\\
:- modeb(*,decrement(+element,-element)).\\
:- modeb(*,increment(+element,-element)). \\
:- modeb(*,element(+list,-element)).\\
:- modeb(*,cons(+element,+list,-list)).
\end{lstlisting}

\subsection{Metagol}
\begin{lstlisting}
body_pred(head/2).
body_pred(tail/2).
body_pred(geq/2).
body_pred(empty/1).
body_pred(even/1).
body_pred(odd/1).
body_pred(one/1).
body_pred(zero/1).
body_pred(decrement/2).
body_pred(increment/2). % ONLY FOR SORTED
body_pred(member/2). % ONLY FOR FIND DUPLICATE

metarule([P,Q], [P,A], [[Q,A]]).
metarule([P,Q], [P,A], [[Q,A]]).
metarule([P,Q,R], [P,A], [[Q,A,B],[R,B]]).
metarule([P,Q], [P,A], [[Q,A,B],[P,B]]).
metarule([P,Q,R], [P,A], [[Q,A,B],[R,A,B]]).
metarule([P,Q], [P,A,B], [[Q,A,B]]).
metarule([P,Q,R], [P,A,B], [[Q,A,B],[R,A,B]]).
metarule([P,Q,R], [P,A,B], [[Q,A],[R,A,B]]).
metarule([P,Q,R], [P,A,B], [[Q,A,B],[R,B]]).
metarule([P,Q,R], [P,A,B], [[Q,A,C],[R,C,B]]).
metarule([P,Q], [P,A,B], [[Q,A,C],[P,C,B]]).
metarule([P,Q], [P,A,B], [[Q,A,B]]).
\end{lstlisting}
\end{document}